\documentclass{article}

\usepackage[final,nonatbib]{neurips_2022}

\usepackage[utf8]{inputenc} % allow utf-8 input
\usepackage[T1]{fontenc}    % use 8-bit T1 fonts
\usepackage{hyperref}       % hyperlinks
\usepackage{url}            % simple URL typesetting
\usepackage{booktabs}       % professional-quality tables
\usepackage{amsfonts}       % blackboard math symbols
\usepackage{nicefrac}       % compact symbols for 1/2, etc.
\usepackage{microtype}      % microtypography
\usepackage{xcolor}         % colors
\usepackage{amsmath, amsthm}
\usepackage{mathtools}
\usepackage[capitalize, noabbrev]{cleveref}
\usepackage{enumerate}
\usepackage{pifont}
\usepackage{multirow}
\usepackage{relsize}
\usepackage{accents}
\usepackage{textcomp, gensymb}
\usepackage{caption}
\usepackage{graphicx}
\usepackage{calc}
\usepackage{wrapfig}
\usepackage{titlesec}
\usepackage[linesnumbered, commentsnumbered, ruled, vlined]{algorithm2e}
\usepackage{subcaption}

\newlength{\depthofsumsign}
\Crefname{algocf}{Algorithm}{Algorithms}

\newcommand{\vectorname}[1]{{\mathrm{\mathbf{#1}}}}

\newcommand{\ie}{\textit{i.e.}}
\newtheorem{definition}{Definition}
\newtheorem{proposition}{Proposition}
\newtheorem{lemma}{Lemma}
\newtheorem{axiom}{Axiom}
\newtheorem{identity}{Identity}

\newcommand{\myparagraph}[1]{\vspace{0pt}\noindent{\bf #1}}
\newcommand\thickbar[1]{\accentset{\rule{.4em}{.8pt}}{#1}}

\titlespacing*{\section}{0pt}{0.5ex}{0.5ex}
\titlespacing*{\subsection}{0pt}{0.5ex}{0.5ex}
\title{Relational Proxies: Emergent Relationships as Fine-Grained Discriminators}

\newcommand*{\affaddr}[1]{#1} % No op here. Customize it for different styles.
\newcommand*{\affmark}[1][*]{\textsuperscript{#1}}

\author{%
    Abhra Chaudhuri\affmark[1]
    \hspace{18pt}
   Massimiliano Mancini\affmark[2]
   \hspace{18pt}
   Zeynep Akata\affmark[2,3,4]
   \hspace{18pt}
   Anjan Dutta\affmark[5]\thanks{A. Chaudhuri is with the Department of Computer Science at the University of Exeter. M. Mancini and Z. Akata are with the Cluster of Excellence Machine Learning at the University of Tübingen. A. Dutta is with the Institute for People-Centred AI at the University of Surrey.}\\
\affaddr{\affmark[1] University of Exeter}
\hspace{20pt}
\affaddr{\affmark[2] University of T\"{u}bingen}
\hspace{20pt}
\affaddr{\affmark[3] MPI for Informatics}\\
\affaddr{\affmark[4] MPI for Intelligent Systems}
\hspace{20pt}
\affaddr{\affmark[5] University of Surrey}
}

\begin{document}

\maketitle

\begin{abstract}

    Fine-grained categories that largely share the same set of parts cannot be discriminated based on part information alone, as they mostly differ in the way the local parts relate to the overall global structure of the object. We propose \emph{Relational Proxies}, a novel approach that leverages the relational information between the global and local views of an object for encoding its semantic label. Starting with a rigorous formalization of the notion of distinguishability between fine-grained categories, we prove the necessary and sufficient conditions that a model must satisfy in order to learn the underlying decision boundaries in the fine-grained setting.
    We design Relational Proxies based on our theoretical findings and evaluate it on seven challenging fine-grained benchmark datasets and achieve state-of-the-art results on all of them, surpassing the performance of all existing works with a margin exceeding 4\% in some cases. %\zeynep{it would be great to highlight one more experiment with an intersting conclusion here or give some more details about your theoritical analysis.}
    We also experimentally validate our theory on fine-grained distinguishability and obtain consistent results across multiple benchmarks.
    % \change{We also compare the performance boost achieved by our method on coarse-grained (Tiny ImageNet) and fine-grained (Dogs ImageNet) subsets of ImageNet, and demonstrate that mining cross-view relationships is particularly important in the fine-grained setting.}
    Implementation is available at \url{https://github.com/abhrac/relational-proxies}.
\end{abstract}

\section{Introduction}
\label{sec:intro}
Fine-grained visual categorization (FGVC) primarily requires identifying category-specific, discriminative local attributes \cite{zhang2016SPDACNN, wei2018MaskCNN, lam2017HSNet}. However, the \emph{relationship} of the attributes with the global view of the object is also known to encode semantic information \cite{Choudhury2021UnsupervisedParts, Caron2021DiNo}. Such a relationship can be thought of as the way in which local attributes combine to form the overall object.
When two categories share a large number of local attributes, this cross-view relational information becomes the only discriminator.
To illustrate this in an intuitive example, \cref{fig:model_diagram} shows two fine-grained categories of birds, the White-faced Plover (left and top-right) and the Kentish Plover (bottom-right).
Along with color and texture information, the two categories share a large number of local features like beak, head, body, tail and wings.
% an example when two classes can no longer be separated using their global or local views in isolation. The sooty albatross (left and top-right) and the black-footed albatross (bottom-right) have very similar visual appearance.
% Inappropriate lighting conditions mask out the color and texture information. Individual local parts like shape of the head, body and tail, individual parts of the wings, etc., are also very similar between the two categories.
% Under such a scenario, the only discriminative information lies in the way the individual local parts combine to give rise to the overall global view of the object.
% 
% Under such constraints, the way local parts like the proximal and the distal components of the wing \zeynep{this is such a specific terminology, I am not sure if everyone would understand it. I do not know the meaning of 'proximal' and 'distal' components}, combine to form the global view of the complete wing (by subtending an angle of $<180\degree$ for the sooty albatross and $=180\degree$ for the black-footed albatross), remains as the only discriminative information \zeynep{sentence too long.} \zeynep{we shouldn’t be exploiting object orientation for classification. When the color and the texture is missing, shape could be the discriminative cue, e.g. shape of the beak. could you think of an example in these lines?}.
% 
Given such constraints of largely overlapping attribute sets,
% \zeynep{which constraints are there? you didn't mention any constraints so far}
relational information like the distance between the head and the body, or the angular orientation of the legs with respect to the body remain as the only available discriminators.
% 
% For example, the angle between the proximal and distal components of the wing, which happens to be $<180\degree$ for the sooty albatross and $=180\degree$ for the black-footed albatross. The angle subtended by the head and the body at the neck region is also different for the two birds.
% Also, the angle formed by the head and the body at the neck is $<180\degree$ for the black-footed albatross, and $=180\degree$ for the sooty albatross.
% In other words, relational information between the global and local views serve as the only fine-grained discriminator.
% 
We thus conjecture that the way the global structure (view) of the object arises out of its local parts (views) must be an \emph{emergent} \cite{oconnor2021emergence} property of the object which is implicitly encoded as the cross-view relationship. However, all existing methods that consider both global and local information, do so in a \emph{relation-agnostic} manner, \ie, without considering cross-view relationships (we formalize relation-agnosticity in \cref{sec:method}).

% We hypothesize that when two categories largely share the same set of local attributes and differ only in the way the attributes combine to generate the global view of the object, approaches that do not consider the cross-view relational information, do not capture the full semantic information in an input image \zeynep{sentence too long}.

We hypothesize that when two categories largely share the same set of local attributes and differ only in the way the attributes combine to generate the global view of the object, relation-agnostic approaches do not capture the full semantic information in an input image. To prove our hypothesis, we develop a rigorous formalization of the notion of distinguishability in the fine-grained setting. Via our theoretical framework, we identify the necessary and sufficient conditions that a learner must satisfy to completely learn a distribution of fine-grained categories. Specifically, we prove that a learner must harness 
% \zeynep{some/all/the/this/distinguishing/sufficient?} information from
both view-specific (relation-agnostic) and cross-view (relation-aware) information in an input image. We also prove that it is not possible to design a single encoder that can achieve both of these objectives simultaneously.
% Based on our theoretical findings, we design a learner that separately models the relation-agnostic and relation-aware components in an input image, and embeds the corresponding representations in a metric space by learning class representative vectors that we call Relational Proxies \zeynep{sentence too long}.
Based on our theoretical findings, we design a learner that separately computes metric space embeddings for the relation-agnostic and relation-aware components in an input image, through class representative vectors that we call Relational Proxies.

To summarize, we:
% make the following contributions through this work \zeynep{remove 'make the following contributions through this work'}
(1) provide a theoretically rigorous formulation of the FGVC task and formally prove the necessary and sufficient conditions a learner must satisfy for FGVC, (2) introduce a plug-and-play extension on top of conventional CNNs that helps leverage relationships between global and local views of an object in the representation space for obtaining a complete encoding of the fine-grained semantic information in an input image, (3) achieve state-of-the-art results on all benchmark FGVC datasets with significant accuracy gains.

\section{Related Work}
\label{sec:related}

\myparagraph{Fine-grained visual categorization}
% The importance of learning localized features for the task of fine-grained visual categorization has been validated in many prior works
Prior works have demonstrated the importance of learning localized image features for FGVC \cite{Angelova2013DetectSegment, Zhang2014RcnnFgvc, Lin2015DeepLAC}, with extensions exploiting the relationship between multiple images and between network layers \cite{Luo2019CrossX}.
The high intra-class and low inter-class variations in FGVC datasets can be tackled by designing appropriate inductive biases like normalized object poses \cite{Branson2014PoseNorm} or via more data-driven methods like deep metric learning \cite{Cui2016FgvcDml}.
Analysing part-specific features along with the global context was demonstrated through part detection based on activation regions in CNN feature maps \cite{Huang2016PartStackedCNN, Zhang2021MMAL} or via context-aware attention pooling \cite{Behera2021CAP}.
CNNs can also be modified in novel ways for FGVC by incorporating boosting \cite{Moghimi2016BoostedCNN}, kernel pooling \cite{Cui2017}, or by randomly masking out a group of correlated channels during training \cite{Ding2021CDB}. Vision Transformers \cite{Vaswani2017}, with their ability to attend to specific informative image patches, have also shown great promise in FGVC \cite{Wang2021FFVT, He2022TransFG, Lu2021SwinFGHash}.
To the best of our knowledge, we are the first to provide a rigorous theoretical foundation for FGVC and design a cross-view relational metric learning formulation based on the same.

\myparagraph{Relation modelling in deep learning} Modelling relationships between entities has proven to be a useful approach in many areas of deep learning including deep reinforcement learning \cite{Zambaldi2019RelationRL}, object detection \cite{Hu2018RelationOD}, question answering \cite{Santoro2019RelationQA}, graph representation learning \cite{Battaglia2018RelationGraphs}, few-shot learning \cite{Sung2018RelationFSL} and knowledge distillation \cite{Park2019RKD}. The usefulness of modelling relationships between different views of the same image has been demonstrated in the self-supervised context by \cite{Patacchiola2022RelationalReasoningSSL}.
% However, all the above works either aim to learn the relationship itself as the ultimate end goal, mimic relationships between tuples of points in a representation space, or leverage relationships as some form of similarity metric between transformations of images.
All the above works either leverage or aim to learn relationships between entities, the nature of which is assumed to be known \emph{apriori}. Our work breaks free from such assumptions by modelling cross-view relationships as learnable representations that optimize the end-task of FGVC.
% In all of these situations, the nature of such relationships are assumed to be ``known'' \emph{a priori}. Different from the above works, we propose to model relationships between local and global views of an image with the objective of capturing class-specific semantics, \emph{without} explicitly specifying what the nature of the relationship needs to be, which helps a model fully leverage the label information encoded in a set of local views.
% \anjan{it would be good emphasise the fact that the local to global relationship modelling helps learning fine-grained / localized image features}.

\myparagraph{Proxy-based deep metric learning} Motivated by the fact that pairwise losses for deep metric learning incur a significant computational overhead leading to slow convergence, 
% \anjan{how does the idea of proxies help in faster convergence of deep metric learning?}, 
the idea of using proxies for learning metric spaces was first proposed in \cite{Movshovitz-Attias2017ProxyNCA} and enhanced in \cite{Teh2020ProxyNCApp}.
% Substituting pairwise comparisons with assignment to a fixed set of learnable class proxies reduces the training-time complexity from a large polynomial like $\mathcal{O}(n^2)$ or $\mathcal{O}(n^3)$ to near linear $\mathcal{O}(c.n)$, where $c$ is the number of classes in a dataset, $n$ is the number of train-set datapoints, and $c \ll n$.
Proxies can also be used to emulate properties of pairwise losses by capturing data-to-data relations (instead of just data-to-proxy) leveraging relative hardness of datapoints \cite{Kim2020CVPR}, by making data representations follow the semantic hierarchy inherent in real-world classes \cite{Yang2022HierarchicalProxy}, or by regularizing sample distributions around proxies to follow a non-isotropic distribution \cite{Roth2022NIR}.
% \cite{Kim2020CVPR} aims to capture data-to-data relations (instead of just data-to-proxy) inherent in the ordinal nature of metric spaces learnt via pairwise losses by incorporating a scaling factor for the loss gradients, thereby taking into account the relative hardness of datapoints, and a margin term to maintain a robust separation of proxy neighborhoods. \cite{Yang2022HierarchicalProxy} enforces proxy representations to follow a hierarchical structure to reflect the semantic hierarchy inherent in real-world classes. \cite{Roth2022NIR} regularizes sample distributions to be non-isotropic while being mapped to their corresponding proxies by leveraging normalizing flows.
% In contrast, we propose a novel approach
% \anjan{only ``novel approach'' sounds very generic in this case, think of words involving ``proxy''}
% to learn \emph{relationships} between local and global image attributes in a manner such that they form embeddings in a metric space.
However, all the above works perform proxy-based metric learning directly on data representations. In contrast, our approach is designed to learn class proxies that can be used not only to capture isolated, view specific (local/global) information for the underlying class, but also to learn the cross-view \emph{relationships} such that they form embeddings in a metric space.

\section{Relational Proxies}
\label{sec:method}
% \anjan{Section title is very generic, use the terms `relational proxy anchors'}

%\zeynep{section 3 is too long, it contains too much text. more illustrations and shortening of the section is necessary. Especially the section 3.3 could be shortened, maybe some of that material can be moved to the experiments}

% \zeynep{Please replace all the occurrences of 'in order to' with 'to'. less is more.}

% \anjan{It is better to start with problem definition and formulation, because at this stage it is difficult to understand why do we need to model local to global relationships. At the moment the problem definition itself is clear, but the following two sections are still disjoint.}

\begin{figure}
    \centering
    \includegraphics[width=\textwidth]{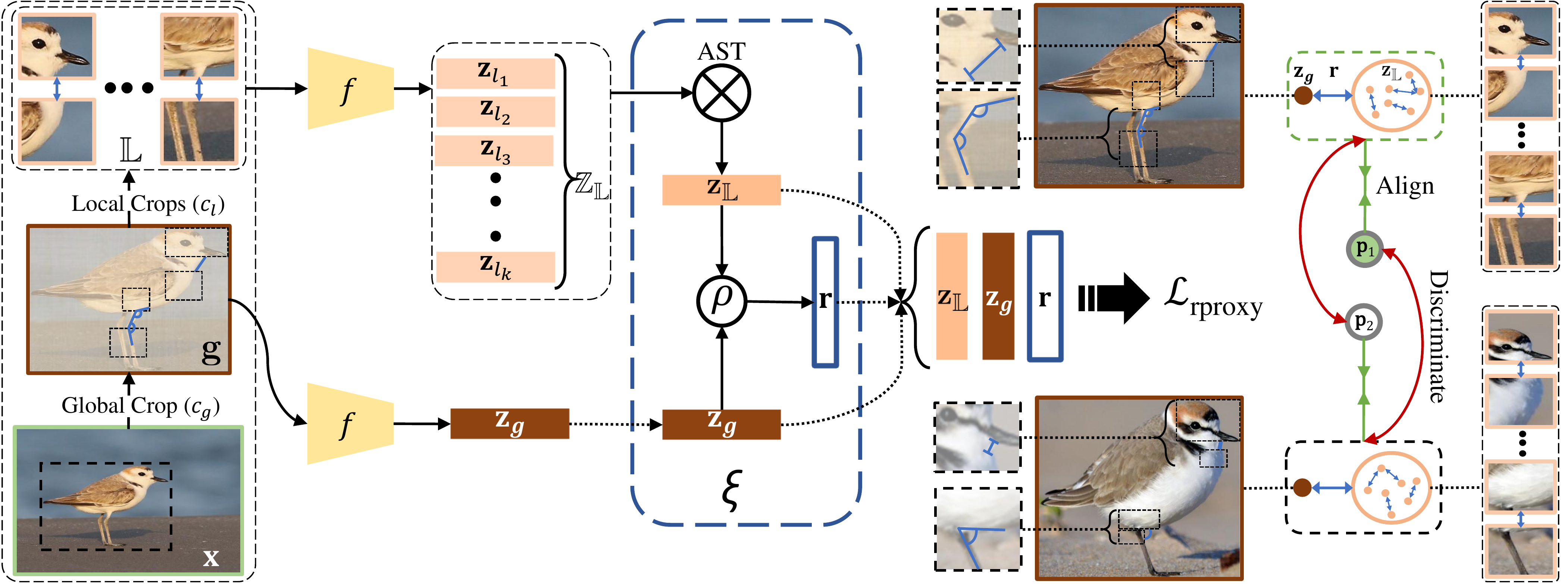}
    \caption{
    % Our Relational Proxies serve as class representatives that capture both view-specific and cross-view relational information.
    We start by encoding the global and local views using a relation-agnostic encoder $f$. We then compute the cross-view relational embedding $\vectorname{r}$ between the global $\vectorname{z}_g$ and the summary of local $\vectorname{z}_\mathbb{L}$ representations. The AST, in conjunction with $\rho$, form the cross-view relational function $\xi$. Finally, the learning of our Relational Proxies is conditioned by both view-specific ($\vectorname{z_\mathbb{L}}$ and $\vectorname{z}_g$) and cross-view relational ($\vectorname{r}$) information.
    % Minimizing $\mathcal{L}_\text{rproxy}$ helps learning representations from the same category close and farther away from representations of different categories in the metric space.
    Minimizing $\mathcal{L}_\text{rproxy}$ helps to align representations from the same category, while discriminating across different categories in a metric space.
    %\zeynep{please name all the boxes that appear in the figure with the notation that appears in the model section (all the images, the pink, brown and blue boxes, all the arrows), the 'align' and 'discriminate' could be denoted in the figure as loss functions, use the same font as the paper also in the figure (times new roman). } %\zeynep{is there any way to show what happens with the 'align' and 'discriminate' operations, e.g. what happens to the embeddings in the embedding space after they go through these operations. it would be very helpful to show that schematically in the model figure. currently the model figure looks like a flowchart without much intuition.}
    % \massi{the font size could be enlarged (e.g. requires a bit of effort to read "align" etc.).} \change{addressed}
    }
    \label{fig:model_diagram}
\end{figure}

%\subsection{Notations} \zeynep{subsection header is not necessary}
Consider an image $\vectorname{x} \in \mathbb{X}$ with a label $\vectorname{y} \in \mathbb{Y}$.
% \massi{currently, we are using "mathbb" to denote both spaces (i.e. $\mathbb{X}$), and sets (i.e. $\mathbb{Y}$) maybe it would be good to distinguish the two of them for clarity and/or clarifying what $\mathbb{X}$ and $\mathbb{Y}$ are by words.}.
Let $\vectorname{g} = c_g(\vectorname{x})$ and $\mathbb{L} = \{\vectorname{l}_1, \vectorname{l}_2,... \: \vectorname{l}_k\} = c_l(\vectorname{x})$ be the global and set of local views of an image $\vectorname{x}$ respectively, where $c_g$ and $c_l$ are cropping functions applied on $\vectorname{x}$ to obtain such views.
% \anjan{I think these $c_g$ and $c_l$ are not entirely different and do have a common part which is the encoder. Does it make sense to model or formulate the notations considering that?} 
% Note, $c_g$ and $c_l$ are just cropping functions and not encoders, producing cropped versions of the input image $\vectorname{x}$.
% At this stage, the encoder has not yet been introduced. The encoder is applied to the outputs of $c_g$ and $c_l$.
Let $f$ be an encoder that takes as input $\vectorname{v} \in \{\vectorname{g}\} \cup \mathbb{L}$ and maps it to a latent space representation $\vectorname{z} \in \mathbb{R}^d$, where $d$ is the representation dimensionality. Specifically, the representations of the global view $\vectorname{g}$ and local views $\mathbb{L}$ obtained from $f$ are then denoted by $\vectorname{z}_g = f(\vectorname{g})$ and $\mathbb{Z_L} = \{f(\vectorname{l}) : \vectorname{l} \in \mathbb{L}\} = \{\vectorname{z}_{l_1}, \vectorname{z}_{l_2},... \: \vectorname{z}_{l_k}\}$ respectively.
% \anjan{what about using $\mathbb{L}$ in the subscript as well?}.
Let $R: (\vectorname{g}, \mathbb{L}) \rightarrow \vectorname{r}$ be a random variable that encodes the relationships $\vectorname{r}$ between the global ($\vectorname{g}$) and the set of local ($\mathbb{L}$) views.

\subsection{Problem Definition}
We leverage the qualitative consistency in the definition of the fine-grained visual categorization (FGVC) problem in the relevant literature \cite{Luo2019CrossX, Zhang2021MMAL, He2022TransFG, Behera2021CAP} to formalize the same in more quantitative terms as follows.

\begin{definition}[\textbf{k-distinguishability}]
    Two categories $\mathcal{C}_1$ and $\mathcal{C}_2$ are said to be $k$-distinguishable iff along with the global view, a classifier needs at least $k$ local features to tell them apart, \ie, the true hypothesis can only distinguish between $\mathcal{C}_1$ and $\mathcal{C}_2$ if it has access to the complete set $\{\vectorname{z}_g\} \cup \mathbb{Z_L}$, and it fails to distinguish between $\mathcal{C}_1$ and $\mathcal{C}_2$, if it only has access to  $\{\vectorname{z}_g\} \cup \mathbb{Z_L} \backslash \vectorname{z}_l, \forall \vectorname{z}_l \in \mathbb{Z_L}$.
    \label{def:k-dis}
\end{definition}

The notion of \emph{k-distinguishability} formalizes what it means for two categories to only be distinguishable in the fine-grained but not in the coarse-grained setting. Given the concept of k-distinguishability, the definition of FGVC problem directly follows from here: %\zeynep{it would be great to illustrate this in a wrapfigure.}:

\begin{definition}[\textbf{Fine-Grained Visual Categorization Problem} - $\mathcal{P}_\text{FGVC}$]
    A categorization problem is said to belong to the $\mathcal{P}_\text{FGVC} \,$ family, iff there exists at least one pair of categories $\mathcal{C}_1$ and $\mathcal{C}_2$ such that they are k-distinguishable.
    \label{def:fgvc}
\end{definition}

Unless otherwise stated, all datapoints $\left( \vectorname{x}, \vectorname{y} \right)$ are considered to be sampled from $k$-distinguishable categories of an instance of $\mathcal{P}_\text{FGVC}$. In the subsequent sections, we prove that for a learner to completely model the class distribution for an instance of $\mathcal{P}_\text{FGVC}$, it must, alongside the view specific representations $\vectorname{z}_g$ and $\mathbb{Z_L}$, also learn a function $\xi$ that models the cross-view relationship between the global and the local views.
Thus, a function $\xi$, to model $R$, must satisfy the following properties: (1) \emph{View-Unification}: Maps the set of all views $\{\vectorname{g}, \vectorname{l} \in \mathbb{L}\}$ of an image $\vectorname{x}$ to a single output $\vectorname{r}$; (2) \emph{Permutation Invariance}: Produces the same output irrespective of the order of the local attributes, \ie, $\xi(\vectorname{z}_g, \{\vectorname{z}_{l_1}, \vectorname{z}_{l_2},... \: \vectorname{z}_{l_k}\}) = \xi(\vectorname{z}_g, \{\vectorname{z}_{l_{\pi(1)}}, \vectorname{z}_{l_{\pi(2)}},... \: \vectorname{z}_{l_{\pi(k)}}\})$, for every permutation $\pi$, where $\vectorname{z}_g$ and $\vectorname{z}_{l_i}$ are the representations of the global and the local views respectively, obtained from $f$. We provide more details on the necessity of these properties in \cref{subsec:props_of_xi} of the Appendix.

\subsection{Relation-Agnostic Representations and Information Gap}
In this section, we formally study the nature of the representation spaces learned by models that do not consider the cross-view
% \anjan{it would be good to introduce to the readers with the notion of cross-view which in this case means the local-global view}
relational information in the context of $\mathcal{P}_\text{FGVC}$. We term such representations as being "relation-agnostic" and prove via \cref{prop:Ixy_Izy_gap} that they suffer from an Information Gap, and thus are unable to capture the complete label information encoded in an input image.
% Here, we denote the family of learners that do not consider cross-view relational information as ones that produce "relation-agnostic representations", and formally study the nature of such representation spaces in the context of .

\begin{definition}[\textbf{Relation-Agnostic Representations - Information Theoretic}]
\label{def:RelAgnRepr}
An encoder is said to produce relation-agnostic representations if it independently encodes the global view $\vectorname{g}$ and local views $\vectorname{l} \in \mathbb{L}$ of $\vectorname{x}$ without considering their relationship information $\vectorname{r}$.
\end{definition}

\begin{lemma}
Given a relation-agnostic representation $\vectorname{z}$ of $\vectorname{x}$, the conditional mutual information between $\vectorname{x}$ and $\vectorname{y}$ given $\vectorname{z}$ can be reduced to $I(\vectorname{x};\vectorname{r} | \vectorname{z})$.
\label{lemma:RelationalReduction}
\end{lemma}
\begin{proof}

Given a relation-agnostic representation $\vectorname{z}$ of $\vectorname{x}$, the only uncertainty that remains about the label information $\vectorname{y}$ can be quantified as the cross-view relational information $\vectorname{r}$, \ie, $I(\vectorname{x};\vectorname{y} | \vectorname{z}) = I(\vectorname{x};\vectorname{r})$. The proof of this statement is given in \cref{identity:relational_uncertainty_z} of the Appendix.

Intuitively, the conditional mutual information between $\vectorname{x}$ and $\vectorname{y}$ given $\vectorname{z}$, \ie, $I(\vectorname{x};\vectorname{y} | \vectorname{z})$ represents the information for predicting $\vectorname{y}$ from $\vectorname{x}$ that $\vectorname{z}$ is unable to capture. Since $\vectorname{z}$ is relation-agnostic, the only uncertainty that remains in $\vectorname{x}$ after $\vectorname{z}$ is the cross-set relationship between the global and the local views, \ie, $\vectorname{r}$. Therefore, we can write $I(\vectorname{x};\vectorname{y} | \vectorname{z}) = I(\vectorname{x};\vectorname{r})$. Using this equality and further factorizing $I(\vectorname{x};\vectorname{r})$ using the chain rule for mutual information, we get:
    \begin{equation*}
        I(\vectorname{x};\vectorname{y} | \vectorname{z}) = I(\vectorname{x};\vectorname{r}) = I(\vectorname{x};\vectorname{r} | \vectorname{z}) + I(\vectorname{r};\vectorname{z}) = I(\vectorname{x};\vectorname{r} | \vectorname{z}),
    \end{equation*}
    the latter equality following from \cref{def:RelAgnRepr}, which implies that $I(\vectorname{r};\vectorname{z}) = 0$, since $\vectorname{z}$ does not explicitly model the local-to-global relationships $\vectorname{r}$.
\end{proof}

\begin{lemma}
    The mutual information between $\vectorname{x}$ and its relation-agnostic representation $\vectorname{z}$ does not change with the knowledge of $\vectorname{r}$.
    \label{lemma:RelationalExclusion}
\end{lemma}
\begin{proof}
    Following the chain rule \cite{Federici2020MIB}, the mutual information between $\vectorname{x}$ and $\vectorname{z}$, \ie, $I(\vectorname{x};\vectorname{z})$ can be expressed as $I(\vectorname{x};\vectorname{z} | \vectorname{r}) + I(\vectorname{z};\vectorname{r})$. However, since $\vectorname{z}$ is \emph{relation-agnostic} (\cref{def:RelAgnRepr}), $I(\vectorname{z};\vectorname{r}) = 0$. Thus,
    $I(\vectorname{x};\vectorname{z}) = I(\vectorname{x};\vectorname{z} | \vectorname{r}) + I(\vectorname{z};\vectorname{r}) = I(\vectorname{x};\vectorname{z} | \vectorname{r})$.
    % \begin{equation*}
    %     I(\vectorname{x};\vectorname{z}) = I(\vectorname{x};\vectorname{z} | \vectorname{r}) + I(\vectorname{z};\vectorname{r}) = I(\vectorname{x};\vectorname{z} | \vectorname{r})
    % \end{equation*}
\end{proof}

% \begin{lemma}
%     The mutual information between datapoints $x$ and their labels $y$, \ie, $I(x;y)$ can be expressed as a sum of the conditional mutual information between $x$ and $r$ given $z$ and that of $x$ and $z$ given $r$.
% \end{lemma}
% \begin{proof}
%     The mutual information $I(x;y)$ between a set of datapoints $x$ and their ground-truth labels $y$ can be expressed as $I(x;y | z) + I(x;z)$ based the chain rule. Here $I(x;y | z)$ represents the information for predicting $y$ from $x$ that $z$ is unable to capture, while $I(x;z)$ denotes the predictive information that $z$ does capture from $x$. We can thus rewrite $I(x;y)$ using \cref{lemma:RelationalReduction} and \cref{lemma:RelationalExclusion} as:
%     \begin{equation*}
%         I(x;y) = I(x;y | z) + I(x;z) = I(x;r | z) + I(x;z | r)
%     \end{equation*}
% \end{proof}

\begin{proposition}
    For relation-agnostic representation $\vectorname{z}$ of $\vectorname{x}$, the label information encoded in $\vectorname{z}$ is strictly upper-bounded by the label information in $\vectorname{x}$, \ie, $I(\vectorname{x}; \vectorname{y}) > I(\vectorname{z}; \vectorname{y})$ by an amount $I(\vectorname{x}; \vectorname{r} | \vectorname{z})$.
    \label{prop:Ixy_Izy_gap}
\end{proposition}
% An objective function that optimizes for the predictability of $y$ from $z$, such that $z$ is obtained from $x$ by independently encoding $g$ and $l \in L$ without considering the cross-set relationships $r$, will learn representations such that  $I(z; y) < I(x; y)$.
% \textit{An objective function that independently optimizes the performance of $f$ independently on $g$ and $l \in L$ without considering the cross-set relationships $r$, will learn representations $z$ such that  $I(z; y) < I(x; y)$}.
\begin{proof}
    The mutual information $I(\vectorname{x};\vectorname{y})$ between a datapoint $\vectorname{x}$ 
    % \anjan{check $\vectorname{x}$ is an image}
    and its ground-truth label $\vectorname{y}$ can be expressed as $I(\vectorname{x};\vectorname{y} | \vectorname{z}) + I(\vectorname{x};\vectorname{z})$ based on the chain rule. Here $I(\vectorname{x};\vectorname{y} | \vectorname{z})$ represents the information for predicting $\vectorname{y}$ from $\vectorname{x}$ that $\vectorname{z}$ is unable to capture, while $I(\vectorname{x};\vectorname{z})$ denotes the predictive information that $\vectorname{z}$ does capture from $\vectorname{x}$. We can thus rewrite $I(\vectorname{x};\vectorname{y})$ using \cref{lemma:RelationalReduction} and \cref{lemma:RelationalExclusion} as:
    \begin{equation}
        I(\vectorname{x};\vectorname{y}) = I(\vectorname{x};\vectorname{y} | \vectorname{z}) + I(\vectorname{x};\vectorname{z}) = I(\vectorname{x};\vectorname{r} | \vectorname{z}) + I(\vectorname{x};\vectorname{z} | \vectorname{r})
        \label{eqn:RZIndependence}
    \end{equation}

    Now, using the chain rule of mutual information, $I(\vectorname{z};\vectorname{y}) = I(\vectorname{z};\vectorname{y} | \vectorname{x}) + I(\vectorname{z};\vectorname{x})$. However, as a consequence of the data processing inequality \cite{Federici2020MIB}, $I(\vectorname{z};\vectorname{y} | \vectorname{x}) = 0$ (since $\vectorname{z}$ cannot encode any more information about $\vectorname{y}$ than $\vectorname{x}$). Applying this and \cref{lemma:RelationalExclusion} to \cref{eqn:RZIndependence}:
    \begin{equation*}
        I(\vectorname{x};\vectorname{y}) = I(\vectorname{x};\vectorname{r} | \vectorname{z}) + I(\vectorname{x};\vectorname{z} | \vectorname{r}) = I(\vectorname{x};\vectorname{r} | \vectorname{z}) + I(\vectorname{x};\vectorname{z}) = I(\vectorname{x};\vectorname{r} | \vectorname{z}) + I(\vectorname{z};\vectorname{y})
    \end{equation*}
    Therefore, $I(\vectorname{x};\vectorname{y}) > I(\vectorname{z}; \vectorname{y})$, by an amount $I(\vectorname{x}; \vectorname{r} | \vectorname{z})$.
\end{proof}

\myparagraph{Intuition:} By establishing a strict upper-bound, \cref{prop:Ixy_Izy_gap} shows that relation-agnostic encoders cannot fully capture the label information in an input image. The quantity they are unable to capture is given by $I(\vectorname{x}; \vectorname{r} | \vectorname{z})$, which we call the Information Gap.

% \begin{definition} (\textbf{Complete Representation})
%     % Following from \cref{prop:Ixy_Izy_gap}, a representation is said to be complete if and only if $I(\vectorname{x}; \vectorname{y}) = I(\vectorname{z}; \vectorname{y})$. In other words, a complete representation, $I(\vectorname{x}; \vectorname{r} | \vectorname{z}) = 0$.
%     To prove: I(x; y) = H(z; r)
%     I(x; y) = I(z; y) + I(x; r | z)
%     I(x; y) = H(z) + H(r)
    
%     H(z; r) = H(z) + H(r) - I(z;r)
%     Since z is disjoint, I(z; r) = 0.
%     H(z; r) = H(z) + H(r)
    
%     Hence, I(x; y) = H(z; r)
    
% \end{definition}

\subsection{Sufficient Learner}
% In this section, we first explore the necessity of a model to preserve the relation-agnostic nature of a representation space for problems of the $\mathcal{P}_\text{FGVC}$ family. We then prove, that a model producing such representations, although necessary, is not sufficient for concurrently encoding the cross-view relational information. We finally conclude that a sufficient learner must consider both the relation-agnostic and relational information in order to capture the complete label information encoded in an input image. \massi{we might shrink this part}
\cref{prop:Ixy_Izy_gap} states that the information gap exists \emph{if} the representation space happens to be relation-agnostic.
% \emph{if} a learner constructs a relation-agnostic representation space $\vectorname{z}$, \emph{then} it must also account for the cross-view relational information $\vectorname{r}$.
We now explore if there is really the need to learn relation-agnostic representations in the first place. From there, we identify the necessary and sufficient conditions for a complete learning of $I(\vectorname{x}; \vectorname{y})$, and derive the requirements for a learner to do the same.
% We now explore the requirements for designing a learner that satisfies the necessary and sufficient conditions for learning $\vectorname{z}$ and $\vectorname{r}$ simultaneously.

\begin{definition}[\textbf{Relation-Agnostic Representations - Geometric}]
    % Let $\vectorname{z}_g = f(\vectorname{g})$ and $\mathbb{Z}_L = \{f(\vectorname{l}) : \vectorname{l} \in \mathbb{L}\}$. Also, 
    Let $n_\epsilon(\cdot)$ represent the $\epsilon$-neighbourhood around a point in the limit $\epsilon \to 0$ \footnote{The choice of $\epsilon$ determines the degree of relation-agnosticity of the representation space.}.
    A representation space is relation-agnostic if and only if $\forall \vectorname{z}_l \in \mathbb{Z}_L : n_\epsilon(\vectorname{z}_l) \cap n_\epsilon(\vectorname{z}_g) = \phi$.
    \label{def:geometric-disjointness}
\end{definition}
% \begin{definition}[\textbf{k-distinguishability}]
%     Two categories $\mathcal{C}_1$ and $\mathcal{C}_2$ are said to be $k$-distinguishable iff along with the global view, a classifier needs at least $k$ local features to tell them apart, \ie, the true hypothesis can only distinguish between $\mathcal{C}_1$ and $\mathcal{C}_2$ if it has access to the complete set $\{\vectorname{z}_g\} \cup \mathbb{Z}_L$, and it fails to distinguish between $\mathcal{C}_1$ and $\mathcal{C}_2$, $\forall \vectorname{z}_l \in \mathbb{Z}_l, \{\vectorname{z}_g\} \cup \mathbb{Z}_L \backslash \vectorname{z}_l$.
%     \label{def:k-dis}
% \end{definition}

% \begin{definition}[\textbf{Fine-grained classification problem} - $\mathcal{P}_\text{FGVC}$]
%     A classification problem is said to belong to the $\mathcal{P}_\text{FGVC} \,$ family, iff there exists at least one pair of classes $\mathcal{C}_1$ and $\mathcal{C}_2$ such that they are k-distinguishable.
%     \label{def:fgvc}
% \end{definition}
An intuitive explanation of \cref{def:geometric-disjointness} can be found in \cref{subsec:note_geom_rel_agn} of the Appendix.
% \massi{I would consider adding an intuitive explanation after each definition etc. (in such a way that the reader can easily follow them)}
\begin{axiom}
    $f$ learns representations $\vectorname{z}$ such that a classifier operating on the domain of
    % classification head trained on top of
    $\vectorname{z}$ learns a distribution $\hat{\vectorname{y}}$, minimizing its cross-entropy with the true distribution $-\sum_{i} \vectorname{y}_i \log(\hat{\vectorname{y}}_i)$, where $i$ denotes the $i$-th class.
    \label{axiom:minCE}
\end{axiom}

\begin{lemma}
    For an instance of $\mathcal{P}_\text{FGVC} \,$, the representation space learned by $f$ is relation-agnostic, \ie, the global view $\vectorname{g}$ and the set of local views $\vectorname{l} \in \mathbb{L}$ are mapped to disjoint locations in the representation space.
    \label{lemma:fLearnsDisjoint}
\end{lemma}

\begin{proof}
    From \cref{def:geometric-disjointness}, a representation space is not relation-agnostic \emph{iff} $\exists \vectorname{z}_l \in \mathbb{Z}_L : n_\epsilon(\vectorname{z}_l) \cap n_\epsilon(\vectorname{z}_g) \neq \phi$. Under this condition, the classifier only has the information from $\{\vectorname{z}_g\} \cup \mathbb{Z}_L \backslash \vectorname{z}_l$ instead of the required $\{\vectorname{z}_g\} \cup \mathbb{Z}_L$. Thus, for instances of $\mathcal{P}_\text{FGVC}$, according to \cref{def:k-dis}, removing the relation-agnostic nature from the representation space of $f$ would cause a downstream classifier to produce misclassifications across the instances of $k$-distinguishable categories, leading to a violation of \cref{axiom:minCE}. Hence, $f$ can only learn relation-agnostic representations.
    % for problems of the $\mathcal{P}_\text{FGVC}$ family.
\end{proof}
We can thus conclude from \cref{lemma:fLearnsDisjoint} and \cref{prop:Ixy_Izy_gap} that the necessary and sufficient conditions for a learner to capture the complete label information $I(\vectorname{x}; \vectorname{y})$, are to consider both (1) the relation-agnostic information $\vectorname{z}$ and (2) the cross-view relational information $\vectorname{r}$.

\begin{proposition}
    An encoder $f$ trained to learn relation-agnostic representations $\vectorname{z}$ of datapoints $\vectorname{x}$ cannot be used to model the relationship $\vectorname{r}$ between the global and local views of $\vectorname{x}$.
    % A contrastive/proxy-based metric learner that has already converged to unique representations of local and global views cannot be trained further to model their interrelationships.
    \label{prop:separate_relational_model}
\end{proposition}

\begin{proof}
    $f: \vectorname{x}_v \rightarrow \vectorname{z}_v$ is a unary function that takes as input a (global or local) view $\vectorname{x}_v$ of an image $\vectorname{x}$ and produces view-specific (\cref{lemma:fLearnsDisjoint}) representations $\vectorname{z}_v$ for a downstream function $g: \vectorname{z}_v \rightarrow \vectorname{y}$.
    
    For $f$ to model the cross-view relationships, it must output the same vector $\vectorname{r}$ irrespective of whether $\vectorname{x}_v = \vectorname{g}$ or $\vectorname{x}_v = \vectorname{l} \in \mathbb{L}$, \ie~whether $\vectorname{x}_v$ is a global or a local view of the input image $\vectorname{x}$ (\emph{view-unification} property of $\xi$). However, \cref{lemma:fLearnsDisjoint} prevents this from happening by requiring the output space of $f$ to be relation-agnostic. Hence, $f$ cannot be used to model $\vectorname{r}$.
\end{proof}

Thus, to bridge the information gap, a learner must have distinct sub-models that individually satisfy the properties of being relation-agnostic and relation-aware. Only such a learner could qualify as being sufficient for an instance of $\mathcal{P}_\text{FGVC}$.

\myparagraph{Intuition:} In this section, we have effectively proven that the properties of relation-agnosticity and relation-awareness are dual to each other. We show that while relation-agnosticity is not sufficient, it is a necessary condition for encoding the complete label information $I(\vectorname{x}; \vectorname{y})$. We also show that a disjoint encoder cannot be used to model the two properties alone without violating one of the necessary criteria. The requirement of a separate, relation-aware sub-model follows from here.

\subsection{Learning Relation-Agnostic and Relation-Aware Representations}
% \anjan{Please make sure we address the comment of reviewer 1ZBu (reviewer 1). We promised that we would provide a dedicated summary of the model architecture.}

\cref{fig:model_diagram} depicts the end-to-end design of our framework. Derived from our theoretical findings, it comprises of both the relation-agnostic agnostic encoder $f$, and the cross-view relational function, $\xi$, expressed as a composition of the Attribute Summarization Transformer, AST, and a network for view-unification, $\rho$. Below, we elaborate on each of these components.

% \subsection{Methodology \zeynep{this section header is too generic. you could quote the name and the abbreviation of your model as the subsection header}}
\myparagraph{Relation-Agnostic Representations:} 
% \cref{fig:model_diagram} depicts the end-to-end design of our framework.
We follow recent literature \cite{Wei2017SCDA, Zhang2021MMAL} for localizing the object of interest in the input image $\vectorname{x}$ and obtaining the global view $\vectorname{g}$ by thresholding the final layer activations of a CNN encoder $f$ and detecting the largest connected component in the thresholded feature map.
% We obtain the initial set of local views $\{\vectorname{l}_1, \vectorname{l}_2\, \ldots, \vectorname{l}_k\}$ by cropping out five disjoint locations (primarily for learning stability), specifically, the four corners and the center of $\vectorname{g}$. As learning progresses and the representations stabilize, we also use random crops as local views for a more complete exploration of the set of local attributes.
We obtain the set of local views $\{\vectorname{l}_1, \vectorname{l}_2\, \ldots, \vectorname{l}_k\}$ as sub-crops of $\vectorname{g}$ (more details in \cref{subsec:expt_settings_datasets}).
% As required by \cref{prop:Ixy_Izy_gap}, the primary, necessary condition for a learner to completely encode $I(\vectorname{x}; \vectorname{y})$, \ie, the label information $\vectorname{y}$ in $\vectorname{x}$ for an instance of $\mathcal{P}_\text{FGVC}$, is to learn a relation-agnostic representation space.
% 
% As per the information theoretic definition of relation-agnostic representations (\cref{def:DisjointRepr}), relation-agnostic representations can be obtained by independently encoding the global $\vectorname{g}$ and the local $\vectorname{l} \in \mathbb{L}$ views, without considering their relationships.
% Following \cref{lemma:fLearnsDisjoint},
Following the primary requirement of \cref{prop:Ixy_Izy_gap}, we produce relation-agnostic representations by propagating $\vectorname{g}$ and $\vectorname{l}_i$ through a CNN encoder $f$ that independently encodes the two view families as $\vectorname{z}_g$ = $f(\vectorname{g})$ and $\vectorname{z}_{l_i} = f(\vectorname{l}_i)$.

% \mathcal{L}_\text{disjoint} = -\sum_{i} \vectorname{y}_i \log(\hat{\vectorname{y}}_i),
% where $i$ denotes the $i$-th class.
\label{subsec:disj_repr_cls_prx}

\myparagraph{Relational Embeddings:}
% According to \cref{prop:Ixy_Izy_gap}, the secondary requirement to completely learn $I(\vectorname{x}; \vectorname{y})$, \ie, the class information $\vectorname{y}$ in $\vectorname{x}$,
The second requirement, according to \cref{prop:Ixy_Izy_gap}, for completely learning $I(\vectorname{x}; \vectorname{y})$
is to minimize % the conditional mutual information
$I(\vectorname{x}; \vectorname{r} | \vectorname{z})$, \ie, the uncertainty about the relational information $\vectorname{r}$ encoded in $\vectorname{x}$, given a relation-agnostic representation $\vectorname{z}$. However, according to \cref{prop:separate_relational_model}, we cannot perform the same using the relation-agnostic encoder $f$. Contrary to existing relational learning literature \cite{Park2019RKD, Patacchiola2022RelationalReasoningSSL} that assumes the nature of relationships to be known beforehand, we take a novel approach that models cross-view relationships as learnable representations of the input $\vectorname{x}$. We follow the definition of the relationship modelling function $\xi: (\vectorname{g}, \mathbb{L}) \rightarrow \vectorname{r}$, that takes as input relation-agnostic representations of the global view $\vectorname{z}_g$ and the set of local views $\mathbb{Z_L} = \{\vectorname{z_{l_1}, z_{l_1}, ... \: z_{l_k}}\}$, and outputs a relationship vector $\vectorname{r}$, satisfying the \emph{View-Unification} and \emph{Permutation Invariance} properties.

We satisfy the Permutation Invariance property by aggregating the local representations via a novel Attribute Summarization Transformer (AST).
We form a matrix whose columns constitute a learnable summary embedding $\vectorname{z}_\mathbb{L}$ followed by the local representations $\vectorname{z}_{l_i}$ as $\mathbf{Z}_\mathbb{L}' = [\vectorname{z_\mathbb{L}, z_{l_1}, z_{l_1}, ... \: z_{l_k}}]$.
% We prepend a learnable attribute summary embedding  $\vectorname{z}_\mathbb{L}$ to the set of local-view embeddings and obtain $\mathbf{Z}_\mathbb{L}' = \{\vectorname{z_\mathbb{L}, z_{l_1}, z_{l_1}, ... \: z_{l_k}}\}$.
We compute the self-attention output $\vectorname{z}'_*$ for each column $\vectorname{z}_*$ in $\mathbf{Z}_\mathbb{L}'$ as $\vectorname{z}'_* = \vectorname{a} \cdot \mathbf{Z}_\mathbb{L}\mathbf{W}$, where $\vectorname{a} = \sigma \left( (\vectorname{z}_* \textbf{W}_q) \cdot (\mathbf{Z}_\mathbb{L}\mathbf{W})^T / \sqrt{D} \right)$,
% \begin{equation*}
%     \vectorname{a} = \sigma \left( (\vectorname{z}_* \textbf{W}_q) \cdot (\mathbf{Z}_\mathbb{L}\mathbf{W})^T / \sqrt{D} \right),
%     \quad \vectorname{z}'_* = \vectorname{a} \cdot \mathbf{Z}_\mathbb{L}\mathbf{W},
% \end{equation*}
and $D$ is the embedding dimension.
By iteratively performing self-attention operations among the columns of $\mathbf{Z}_\mathbb{L}'$, AST aggregates information across all the local attributes into the final learnable output of $\vectorname{z}_\mathbb{L}$.
% \begin{align*}
%     \forall \vectorname{z}_l \in \mathbb{Z}_L, [\vectorname{q}, \vectorname{k}, \vectorname{v}] &= \vectorname{z}_l \vectorname{U}_{qkv} \\
%     A &= \operatorname{softmax}\left( \vectorname{q}\vectorname{k}^T / \sqrt{D} \right) \\
%     \operatorname{SA}\left( \vectorname{z} \right) &= \left( A\vectorname{v} \right),
% \end{align*}
% \anjan{Too much white space in the above equation}
Unlike the usual vision transformer \cite{Vaswani2017}, we omit the usage of positional embeddings, as doing so provides better permutation invariance \cite{Naseer2021IntriguingViT}.

For satisfying the View-Unification property, we introduce a simple feed-forward multilayer perceptron that learns the mapping $\rho: (\vectorname{z_g}, \vectorname{z}_\mathbb{L}) \rightarrow \vectorname{r}$. It takes as input the representation of the global view $\vectorname{z}_g$ and the summary of the set of local views $\vectorname{z}_\mathbb{L}$, and outputs the relationship as a learned vector $\vectorname{r}$. Thus, in our construction, the AST along with $\rho$, constitute the relation modelling function $\xi$.

\myparagraph{Learning Relational Proxies:} The representations $\vectorname{z}_g, \vectorname{z}_\mathbb{L}$ and $r$ in unison encode the full semantic information $\vectorname{y}$ in $\vectorname{x}$ (\cref{prop:Ixy_Izy_gap}). To alleviate the low inter-class variance in $\mathcal{P}_\text{FGVC}$, metric learning has been shown to be an effective \cite{Cui2016FgvcDml} approach.
% However, the computational cost of computing pairwise losses for learning metric spaces is prohibitive and leads to slow convergence \cite{Movshovitz-Attias2017ProxyNCA}.
Furthermore, approaches like \cite{Movshovitz-Attias2017ProxyNCA} and \cite{Kim2020CVPR} for metric learning have shown that substituting pairwise comparisons with assignment to a fixed set of learnable class proxies reduces the training-time complexity from a large polynomial like $\mathcal{O}(n^2)$ or $\mathcal{O}(n^3)$ to near linear $\mathcal{O}(c.n)$, where $c$ is the number of classes in a dataset, $n$ is the number of train-set datapoints, and $c \ll n$.
For this purpose, we contrast instance representations across classes through class proxy vectors that are informed by both the view specific and relational representations via learning the conditional distribution $p(\vectorname{y} | \vectorname{z}_g, \vectorname{z}_\mathbb{L}, \vectorname{r})$. We term such class proxies, Relational Proxies, as they leverage cross-view relationship information for encoding class semantics.

Consider a set of $c$ learnable class proxy vectors $\mathbb{P} = \{\vectorname{p}_1, \vectorname{p}_2, ... \: \vectorname{p}_c \}$, where $c$ is the number of fine-grained classes.
% in the instance of $\mathcal{P}_\text{FGVC}$ under consideration.
Here, we present a novel formulation of the proxy-anchor loss \cite{Kim2020CVPR} in cross-entropic terms that allows us to conform to the requirement of \cref{axiom:minCE} in the fine-grained setting.
Specifically, for each of the representations $\omega \in \{\vectorname{z}_g, \vectorname{z}_l, \vectorname{r}\}$ for all $\vectorname{x} \in \mathbb{X}$, we minimize the following:
% \begin{equation}
%     \centering
%     % \resizebox{\textwidth}{!}{$
%     \mathcal{L}_\text{rproxy} = - \frac{1}{c} \sum_{\vectorname{p}_i \in \mathbb{P}} \vectorname{y}_{\vectorname{p}_i} \log \left( \frac{1}{\psi^+_i \cdot \psi^-_i} \right),
%     \;\;
%     \begin{cases}
%         \psi^+_i = 1 + \displaystyle\sum_i e^{-\alpha(s(\omega_i, \vectorname{p}_i) - \delta)} \\
%         % \; \vectorname{z} = f(\vectorname{x}), \vectorname{x} \in \mathbb{X}_{\vectorname{p}^+} \\
%         \psi^-_i = 1 + \displaystyle\sum_j e^{\alpha(s(\omega_j, \vectorname{p}_i) + \delta)} \\
%         % \; \vectorname{z} = f(\vectorname{x}), \vectorname{x} \in \mathbb{X}_{\vectorname{p}^-}
%     \end{cases}
% \end{equation}
\begin{equation}
    \centering
    \mathcal{L}_\text{rproxy} = - \frac{1}{c} \sum_{\vectorname{p} \in \mathbb{P}} \log \left( \frac{1}{\psi^+ \cdot \psi^-} \right),
    \;\;
    \begin{cases}
        \psi^+ = 1 + \displaystyle\sum_{\omega \in \Omega} e^{-\alpha(s(\omega, \vectorname{p}) - \delta)} \\
        \psi^- = 1 + \displaystyle\sum_{\omega \in \thickbar{\Omega}} e^{\alpha(s(\omega, \vectorname{p}) + \delta)} \\
    \end{cases}
\end{equation}
where $\Omega$ is the set of representations in a mini-batch for which $\vectorname{p}$ is the true class proxy, $\thickbar{\Omega}$ is one for which $\vectorname{p}$ is not the true class proxy,
% where $\omega_i$ is a representation of a datapoint for which $\vectorname{p}$ is the true class proxy, $\omega_j$ is one for which $\vectorname{p}$ is not the true class proxy,
and $s(\cdot, \cdot)$ computes the cosine distance.
% , and $\vectorname{y}_\vectorname{p}$ is the true probability distribution corresponding to the proxy $\vectorname{p}_i$, with all mass on the $i$-th element.
% It can also be interpreted as the label for the datapoints for which the true representative proxy is $p_i$.
$\psi^+$ helps align matching $(\omega, \vectorname{p})$ pairs close together in the representation space (since $\mathcal{L}_\text{rproxy}$ follows a cross-entropic form, it does not violate the relation-agnosticity of $f$, as proven in \cref{lemma:fLearnsDisjoint}, with a more detailed note in \cref{subsec:proxy_relation_agnosticity)} of the appendix), while $\psi^-$ helps embedding non-matching $(\omega, \vectorname{p})$ pairs farther apart. The scaling parameter $\alpha$ along with the margin parameter $\delta$ control the intensity with which the alignment and discrimination are performed.
$1 / (\psi^+ \cdot \psi^-)$ gives a probability indicating how closely the learned representation space reflects the semantic structure in $\mathbb{X}$.
$\mathcal{L}_\text{rproxy}$ thus computes the cross-entropy loss between the ground-truth and the predicted class distributions over the set of proxies.

\myparagraph{Inference:} Given an input image $\vectorname{x}$, we compute its global ($\vectorname{z}_g$), summary of local ($\vectorname{z}_\mathbb{L}$), and relational ($\vectorname{r}$) representations using $f$, AST and $\rho$ as explained above. We then predict the class probability distribution $\hat{\vectorname{y}}$ of these representations by computing their soft-assignment scores across the relational proxies. The assignment score for each proxy $\vectorname{p} \in \mathbb{P}$ is computed as follows:
% We compute the final scores of the input sample $\vectorname{x}$ by computing its assignment score distribution to the set of all class proxies:
% The elements of the class score distribution $\vectorname{\hat{y}}$ of $\vectorname{x}$ is given by:
% \begin{equation*}
%     \vectorname{\hat{y}}_i = \mathlarger{\sum}_{\omega \in \{\vectorname{z}_L, \vectorname{z}_g, \vectorname{r}\}} \frac{e^{s(\omega, \vectorname{p}_i)}}{\displaystyle\sum_{\vectorname{p_j} \in \mathbb{P}} e^{s(\omega, \vectorname{p}_j)}}
%     % \mathcal{L}_\text{disjoint} = -\sum_{i} \vectorname{y}_i \log(\hat{\vectorname{y}}_i),
% \end{equation*}
\begin{equation*}
    \vectorname{\hat{y}}_\vectorname{p} = \mathlarger{\sum}_{\omega \in \{\vectorname{z}_L, \vectorname{z}_g, \vectorname{r}\}} \frac{e^{s(\omega, \vectorname{p})}}{\displaystyle\sum_{\vectorname{p}' \in \mathbb{P}} e^{s(\omega, \vectorname{p}')}}
    % \mathcal{L}_\text{disjoint} = -\sum_{i} \vectorname{y}_i \log(\hat{\vectorname{y}}_i),
\end{equation*}
% where $i$ denotes the $i$-th class.
The class corresponding to the relational proxy with the highest assigned score is returned as the prediction.

\section{Experiments}
\label{sec:expt}

% \zeynep{briefly summarize what is coming in this section}
We now present the implementation details of Relational Proxy, and the results obtained upon evaluating it on benchmark FGVC datasets. We also discuss observations from ablation studies that we performed to validate our theoretical foundations, as well as the implementation specific choices that we made, along with qualitative visualizations of the learned cross-view local relationships.
% \anjan{Since we have 10 content pages now, we can show some visualizations in the main paper.}

% \subsection{\zeynep{Experimental Details: Parameter Setting and Datasets}}
\subsection{Experimental Settings and Datasets}
\label{subsec:expt_settings_datasets}
\myparagraph{Implementation details} -- We implement our Relational Proxy model using the PyTorch \cite{Paszke2017PyTorch} deep learning framework, on an Ubuntu 20.04 workstation with a single NVIDIA GeForce RTX 3090 GPU, an 8-core Intel Xeon processor and 32 GBs of RAM. Since we proposed a proxy-based approach for learning the relational metric space, we do not have a dependency on batch-size for the purpose of negative sampling as part of our metric learning phase, which enables us to train our entire model end-to-end on a single GPU. Also by the virtue of using class-proxies, the convergence time is reduced by a significant amount compared to pairwise losses.

\myparagraph{Hyperparameter settings} -- For initial training stability, we consider five disjoint locations (four corners and the centre) of $\vectorname{g}$ to be the set of local views. As training progresses, we also allow the model to learn from an increased number views obtained via random cropping. In the same way at inference time, the local views constitute a combination of the five disjoint crops along with some random crops.
We found that the optimal number of local views $k$ to be equal to $7$ for FGVC Aircraft, Stanford Cars and both the cultivar datasets. For CUB and NABirds, $k=8$ gave the best performance.
We use ResNet50 \cite{He2016DeepRL} pretrained on ImageNet \cite{Deng2009ImageNetAL} as the backbone of our relation-agnostic encoder $f$. 
% \anjan{Shouldn't we mention that results with VGG-16 backbone are in the supplementary?} 
In Sec. 1.3 of the supplementary, we also provide evaluations using VGG-16 \cite{Simonyan2015VGG16} to show that the performance gains achieved by our model do not depend on the specific backbone. We train our full Relational Proxy model end-to-end for 200 epochs using the stochastic gradient descent optimizer with an initial learning rate of 0.001 (decayed by a factor of 0.1 every 50 epochs), a momentum of 0.9, and a weight decay of $10^{-4}$.

\myparagraph{Datasets and Evaluation} --
We evaluate our model on the four most common fine-grained visual categorization benchmarks (number of classes and train/test splits respectively in brackets): FGVC Aircraft \cite{Maji2013FGVCAircraft} (100 | 6667/3333), Stanford Cars \cite{krause2013StanfordCars} (196 | 8144/8041), CUB \cite{Wah2011CUB} (200 | 5994/5794), and NA Birds \cite{horn2015NABirds} (555 | 23,929/24,633).
% and iNaturalist 2017 \cite{Horn2018iNat} (5089 | 675,170/182,707).
% \anjan{Mention the challenge of iNaturalist dataset that it is a large dataset with many different super and fine-grained categories.}
For large scale benchmark evaluation, we choose the iNaturalist 2017 dataset which consists of 13 super-categories that have been split into a total of 5089 fine-grained categories with 675,170 training and 182,707 test images. We also perform experiments on two challenging datasets of the cultivar domain that offer very low inter-class variations, namely Cotton Cultivar \cite{Yu2020SoyCottonCultivar} (80 | 240/240) and Soy Cultivar \cite{Yu2020SoyCottonCultivar} (200 | 600/600). We use classification accuracy as our metric for evaluating the performance of a model.

\begin{table}[t]
    \centering
    \resizebox{\textwidth}{!}{
    \begin{tabular}{l|c|c|c|c|c|c|c}
         \hline
         \multirow{2}{*}{\textbf{Method}} & \multicolumn{5}{c|}{\textbf{Benchmark}} & \multicolumn{2}{c}{\textbf{Cultivar}} \\
         \cline{2-8}
         & \textbf{FGVC Aircraft} & \textbf{Stanford Cars} & \textbf{CUB} & \textbf{NA Birds} & \textbf{iNaturalist} & \textbf{Cotton} & \textbf{Soy} \\
         \hline
         MaxEnt \cite{dubey2018MaxEntFGVC} NeurIPS'18 & 89.76 & 93.85 & 86.54 & - & - & - & - \\
         DBTNet \cite{zheng2019DBT} NeurIPS'19 & 91.60 & 94.50 & 88.10 & - & - & - & - \\
         StochNorm \cite{kou2020StochNorm} NeurIPS'20 & 81.79 & 87.57 & 79.71 & 74.94 & 60.75 & 45.41 & 38.50 \\

         MMAL \cite{Zhang2021MMAL} MMM'21 & 94.70 & 95.00 & 89.60 & 87.10 & 69.85 & 65.00 & 47.00 \\
        %  \hline
         FFVT \cite{Wang2021FFVT} BMVC'21 & 79.80 & 91.25 & 91.65 & 89.42 & 70.30 & 57.92 & 44.17 \\
        %  \hline
         CAP \cite{Behera2021CAP} AAAI'21 & 94.90 & 95.70 & 91.80 & 91.00 & - & - & - \\
        %  \hline
         TransFG \cite{He2022TransFG} AAAI'22 & 80.59 & 94.80 & 91.70 & 90.80 & 71.70 & 45.84 & 38.67 \\
        %  \hline
         \textbf{Ours (Relational Proxy)} & \textbf{95.25} \textpm~0.02 & \textbf{96.30} \textpm~0.04 & \textbf{92.00} \textpm~0.01 & \textbf{91.20} \textpm~0.02 & \textbf{72.15} \textpm~0.03 & \textbf{69.81} \textpm~0.04 & \textbf{51.20} \textpm~0.02 \\
         \hline
    \end{tabular}}
    \caption{
    Comparison of classification accuracies obtained by our method (averaged over 5 independent runs) on standard FGVC datasets with current state-of-the-art approaches.
    % zeynep{the numbers are too small. remove 20 from 2018, remove paranthesis etc. rewrite it as, e.g. NeurIPS'18. You categorize your experiments as Benchmark and Cultivar in the section below, in the table group them the same way (another line), remove the word 'cultivar' from the last two columns (as it will appear above). all this will increase the font of your numbers. if there are other ways, please implement them, e.g. NA Birds -- NAB? etc}
    % Classification accuracy comparison with the state-of-the-art on benchmark FGVC datasets.
    % \anjan{Try to run MMAL, TransFG and CAP on `Cotton and Soy Cultivar', and FFVT on `FGVC Aircraft', `Stanford Cars' and `NA Birds'}
    }
    \label{tab:sota_comparison}
\end{table}

\subsection{Comparison with State of the Art}
\textbf{Benchmark Datasets} -- In \cref{tab:sota_comparison}, we report the performance of our method on benchmark datasets along with existing SotA approaches. StochNorm \cite{kou2020StochNorm} presents a novel way to refactor batch normalization that helps prevent overfitting for the task of FGVC. The novel training routine proposed in MaxEnt\cite{dubey2018MaxEntFGVC} improves FGVC performance by maximizing the entropy of the output probability distribution of a CNN. By designing a computationally inexpensive bilinear feature transformation mechanism for CNNs, DBT \cite{zheng2019DBT} achieves competitive performance on benchmark FGVC datasets. MMAL \cite{Zhang2021MMAL} is one of the most competitive models for FGVC Aircraft and Stanford Cars, which extracts the most informative global and local views by analyzing the activation maps of the final layer of a CNN, and embeds them in a relation-agnostic representation space.
TransFG \cite{He2022TransFG} proposes a vision transformer based technique for extracting informative local patches, achieving SotA performance on iNaturalist, and promising results on CUB and NA Birds. By learning a context aware attention pooling mechanism, CAP \cite{Behera2021CAP} reports SotA performance on all benchmark datasets other than iNaturalist.
% \cref{tab:sota_comparison} reports the performance of our method, which
From \cref{tab:sota_comparison}, we see that our method
surpasses the SotA on all four benchmarks by significant margins. Specifically, we beat the SotA on Stanford Cars by $0.60\%$, on iNaturalist by 0.45\%, on FGVC Aircraft by $0.35\%$, and on both CUB and NA Birds by $0.20\%$.

Cars and Aircrafts can largely vary in color, texture and custom, part-specific styles within a category. However, the geometry of the overall object (represented by cross-view relationships) within a class remains fairly constant. This leaves room for a large amount of relational information to be captured. This also holds true for the iNaturalist dataset, as the local-to-global emergent relationships can be used to discriminate between both coarse-grained (super) and fine-grained (sub) categories. For the bird datasets (CUB, NABirds), although this relational information is still there, most categories can be told apart by color, texture and local-attribute specific information, if they are clearly visible. For this reason, the accuracy gains obtained in the Cars and Aircraft datasets surpass those obtained for the birds.

% \anjan{Discuss briefly the state-of-the-art results we obtained on the iNaturalist dataset.}

% Most notably, for the highly challenging Cotton and Soy Cultivar datasets, our model manages to provide a performance boost exceeding $2\%$ over FFVT \cite{Wang2021FFVT}, a SOTA technique that proposed a feature fusion technique for vision transformers.

\textbf{Cultivar Datasets} -- For the highly challenging datasets of the cultivar domain, \ie, Cotton and Soy Cultivar, FFVT \cite{Wang2021FFVT} provides state-of-the-art results by using a specialized feature fusion technique for vision transformers. As can be seen in \cref{tab:sota_comparison}, our model, by leveraging cross-view relational embeddings, manages to provide a performance boost exceeding $4\%$ over the current SOTA on the cultivar datasets.
Cultivar datasets have very low inter-class differences. Cross-view relational information like edge curvature, relative angles between leaf sub-parts, width to height ratio, convergence patterns of leaf ends, etc., largely determine the uniqueness of a category. For this reason, our method is extremely effective when applied to such domains.

\begin{table}[t]
    \centering
    \resizebox{\textwidth}{!}{
    \begin{tabular}{c|c|c|c|c|c|c|c|c}
         \hline
         \multirow{2}{*}{ID} & Relation-Agnostic & \multirow{2}{*}{AST} & \multirow{2}{*}{RelationNet} & Learnable & \multirow{2}{*}{Proxies} & \multirow{2}{*}{Aircraft} & \multirow{2}{*}{CUB} & Stanford \\
         & Encoder & & & Relation & & & & Cars \\
         \hline
         1. & \ding{51} & & & & & 94.60 & 91.25 & 95.21 \\
        %  \hline
         2. & \ding{51} & \ding{51} & & & & 94.91 & 91.50 & 95.62 \\
         \hline
         3. & \ding{51} & \ding{51} & \ding{51} & \ding{51} & & 95.13 & 91.90 & 96.15 \\
         \hline
         4. & \ding{51} & & & \ding{51} & \ding{51} & 94.92 & 91.55 & 95.70 \\
        %  \hline
         5. & \ding{51} & \ding{51} & & \ding{51} & \ding{51} & 95.10 & 91.81 & 96.05 \\
        %  \hline
         6. & \ding{51} & & \ding{51} & \ding{51} & \ding{51} & 95.05 & 91.73 & 95.93 \\
        %  \hline
         7. & \ding{51} & \ding{51} & \ding{51} & \ding{51} & \ding{51} & \textbf{95.25} & \textbf{92.00} & \textbf{96.30} \\
         \hline
    \end{tabular}}
    \caption{Results of ablating the key components of our Relational Proxy model (sufficient learner).
    Grouped so as to better illustrate the effect of learning cross-view relationships.
    Note that the meaning of non-existence of a component may vary according to context / other row elements. Refer to the corresponding paragraph in \cref{subsec:ablation} of the main text for further details.
    % \zeynep{the numbers are too small, Relation Agnostic Encoder and Learnable Relation columns occupy too much space. Please shorten or have multiple lines for those columns}
    }
    \label{tab:ablation_sufficiency}
\end{table}

\begin{table}[t]
    \centering
    \resizebox{\textwidth}{!}{
    \begin{tabular}{c|c|c|c|c|c|c}
        \hline
        \multirow{2}{*}{ID} & Attribute & Global & Relational & \multirow{2}{*}{FGVC Aircraft} & \multirow{2}{*}{CUB} & \multirow{2}{*}{Stanford Cars}\\
        & Summary ($\vectorname{z}_\mathbb{L}$) & Representation ($\vectorname{z}_g$) & Representation ($\vectorname{r}$) &  &  &  \\
         \hline
         1. & \ding{51} & \ding{51} & & 94.91 & 91.50 & 95.62 \\
        %  \hline
         2. & \ding{51} & & \ding{51} & 94.85 & 91.58 & 95.75 \\
        %  \hline
         3. & & \ding{51} & \ding{51} & 94.60 & 91.47 & 95.51 \\
        %  \hline
         4. & \ding{51} & \ding{51} & \ding{51} & \textbf{95.25} & \textbf{92.00} & \textbf{96.30} \\
         \hline
    \end{tabular}}
    \caption{Results of ablating the Proxy Conditioning Representations. Note that only the output representation vectors were ablated here, and not the entire model component producing them. The latter has been studied independently with findings reported in \cref{tab:ablation_sufficiency}.}
    \label{tab:ablation_representations}
\end{table}

\begin{table}[t]
\begin{minipage}{0.5\linewidth}
    \centering
    \resizebox{\textwidth}{!}{
    \includegraphics{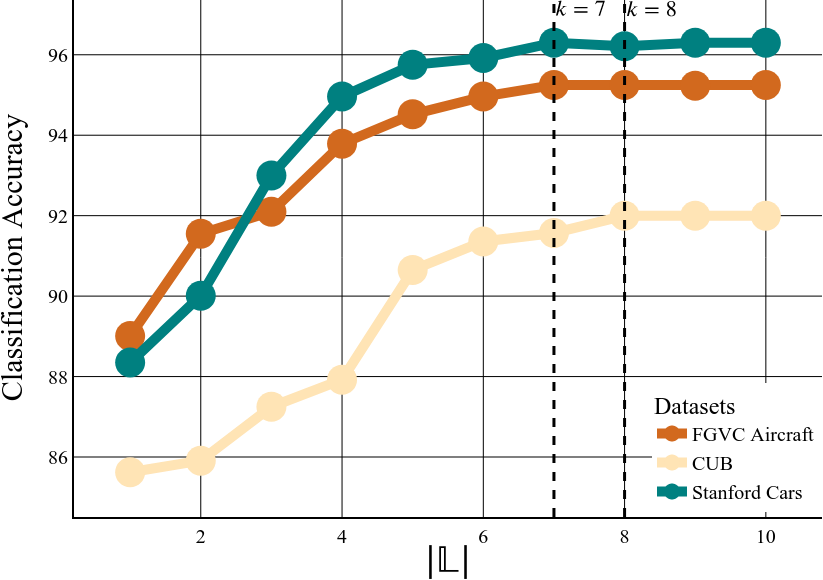}
    }
    \captionof{figure}{Effect of varying the number of local views $|\mathbb{L}|$
    % \zeynep{please provide grid lines in the plot. the text should be black (now it is some shade of gray I think), the font should be the same as in the paper (times new roman).}
    % Effect of varying the number of local views $|\mathbb{L}|$ from 1 to 10 for different datasets.
    % The value of $k$ satisfying the $k$-distinguishability criterion was found to be at $|\mathbb{L}| = 7$ for FGVC Aircraft and Stanford cars, while for CUB, it was at $|\mathbb{L}| = 8$.
    }
    \label{fig:k_values}
\end{minipage}
\hspace{1pt}
\begin{minipage}{0.45\linewidth}
    \centering
    \resizebox{\textwidth}{!}{
    \begin{tabular}{c|c|c}
        \hline
         Method & T-ImageNet & D-ImageNet \\
         \hline
         $\vectorname{z}_g, \vectorname{z}_\mathbb{L}$ & 88.75 & 91.30 \\ % Relation-Agnostic Encoder
         $\vectorname{z}_g, \vectorname{z}_\mathbb{L}, \vectorname{r}$ & 88.91 & 92.75 \\ % Relational Proxy (Ours)
         \hline
         $\Delta$ & 0.16 & \textbf{1.45} \\
         \hline
    \end{tabular}}
    \caption{
    % Comparison of coarse-grained (Tiny ImageNet) \textit{vs.} fine-grained (Dogs ImageNet) accuracy gains obtained
    Comparison of accuracy gains obtained by using the relational information $\vectorname{r}$ along with the view-specific representations $\vectorname{z}_g$ and $\vectorname{z}_\mathbb{L}$ (second row) over using only relation-agnostic representations (first row)
    on the coarse-grained Tiny (T) and fine-grained Dogs (D) subsets of ImageNet. As can be observed, the performance gain $\Delta$ by accounting for the relational information is significantly more in the fine-grained setting than in the coarse-grained one.
    % \zeynep{caption very long but still which line is which setting and what all the z's correspond to is not clear (the table caption should be self explanatory)}
    }
    \label{tab:coarse_vs_fine_grained}
    \end{minipage}
\end{table}

\subsection{Ablation Studies}
\label{subsec:ablation}
We perform the following three classes of ablation studies:
% \begin{enumerate}

% \anjan{it would be good to explain the columns of \cref{tab:ablation_sufficiency}}

\textbf{Key components of the sufficient learner} -- \cref{tab:ablation_sufficiency} shows the results of ablating the key components of our model. The relation agnostic encoder being the most fundamental component, cannot be removed, and therefore appears in all the rows. Row 1 thus represents training a simple classification head on top of the representations obtained from the relation-agnostic encoder. Row 2 denotes the result of aggregating the local views, computing a predefined relationship function, specifically the distance between the local and global representations, and minimizing a Huber loss between the relational distance value between instances of the same class. Row 3 introduces the idea of learnable relational vectors (instead of predefined functions like distances). Since cross-view relationships are unique to a class, we aim to embed the relational vectors in a metric space by minimizing a pairwise contrastive loss across classes. However, as noted in recent metric learning literature \cite{Movshovitz-Attias2017ProxyNCA, Kim2020CVPR}, computing pairwise losses can be expensive and lead to slower convergence. Motivated by this, we introduce the idea of relational proxies in Rows 4 - 7. Row 4 replaces the AST and RelationNet by simple concatenation of the inputs and propagation through a linear layer. Row 5 and 6 individually show the effects of replacing the AST and RelationNet with linear layers. Finally, Row 7 denotes the performance of our model with all components included. 
    
Rows 2 and 3 demonstrate the importance of modelling cross-view relationships specifically as a learnable metric space embedding. Rows 5, 6 and 7 show the contribution of our AST in summarizing the local attributes, as well as the fact that the cross-view relationship is non-linear in nature.
    
\textbf{Conditioning of the relational proxies} -- The relational proxies in our model are conditioned by three representations of the input $\vectorname{x}$ -- the summary of the local attributes $\vectorname{z}_\mathbb{L}$, the representation of the global view $\vectorname{z}_g$, and the relational vector $\vectorname{r}$. We study the contribution of each of these representations and summarize our findings in \cref{tab:ablation_representations}. These results demonstrate that the information encoded in all three representations are necessary for learning the complete set of class attributes.

\textbf{Results on ImageNet subsets} -- In order to validate whether our findings are in fact particularly applicable to the fine-grained setting, we perform experiments to compare the performance boost provided by our method over a vanilla relation-agnostic encoder, between coarse-grained (Tiny ImageNet\cite{Li2017TinyImageNet}) and fine-grained (Dogs ImageNet / Stanford Dogs \cite{Khosla2011StanfordDogs}) subsets of ImageNet.
Our findings are summarized in \cref{tab:coarse_vs_fine_grained}, which shows that our method does in fact provide a more significant improvement over a relation-agnostic encoder in the fine-grained setting.

\textbf{The optimal value of $k$ for the $k$-distinguishability criterion} -- \cref{fig:k_values} empirically illustrates the idea of $k$-distinguishability for a given local-crop size on the FGVC Aircraft, Stanford Cars and CUB datasets. For an instance of $\mathcal{P}_\text{FGVC}$, the performance of a model is strongly dependent on the number of local views it has access to. When the number of local views $|\mathbb{L}|$ is less than the minimum required number $k$, the classification performance is poor as the model does not have access to the minimum set of required fine-grained information. As $|\mathbb{L}|$ approaches $k$, the performance increases, reaching its maximum at $k$ (= 7 for FGVC Aircraft and Stanford Cars, and 8 for CUB). However, if $|\mathbb{L}|$ is increased beyond $k$, there is no further gain in performance, as the extra information is either redundant or semantically irrelevant.
% \change{In \cref{sec:L_vs_patch_size} of the Appendix, we present a correlation study between $\mathbb{L}$ and the size of local patches, to identify the right computational trade-offs for our method.}
% \anjan{I think we can bring the correlation study in the main paper. This was an interesting experiment.} \change{Addressed}

\myparagraph{Correlation between $|\mathbb{L}|$ and local patch size} -- To determine the right computational trade-offs for our method, we perform a study to identify possible correlations between the number of local views $\mathbb{L}$ and size of local patches.
We trained our model on FGVC Aircraft \cite{Maji2013FGVCAircraft} by varying the number of local views $|\mathbb{L}|$ and the size of each local patch to identify their correlations. We present our results in \cref{tab:L_vs_patch_size}, where rows represent the number of local views $|\mathbb{L}|$ and the columns represent the side-length of each local patch. So, if the global view has spatial dimensions $N \times N$, each local patch would be of $N/t \times N/t$, where $t$ is the scaling factor that is varied across the columns. In summary, the rows represent increasing the number of local views top-down, and the columns represent increasing the patch-size left-to-right. The numbers are expressed as relative deviations from a reference of 95.25\%, i.e., the setting corresponding to our reported accuracy for FGVC Aircraft in \cref{tab:sota_comparison}.

\begin{wraptable}[9]{r}{0.4\textwidth}
    \centering
    \begin{tabular}{c|c|c|c|c}
         \hline
         $|\mathbb{L}|$ & $\mathbf{N/5}$ & $\mathbf{N/4}$ & $\mathbf{N/3}$ & $\mathbf{N/2}$ \\
         \hline
         $\mathbf{7}$ & -0.03 & -0.02 & 0.00 & -0.14 \\
         $\mathbf{12}$ & +0.01 & +0.02 & 0.00 & -0.11 \\
         $\mathbf{15}$ & +0.05 & +0.03 & +0.01 & -0.11 \\
         $\mathbf{18}$ & +0.05 & +0.03 & +0.00 & -0.10 \\
         \hline
    \end{tabular}
    \caption{Correlation between number of local views and size of local patches}
    \label{tab:L_vs_patch_size}
\end{wraptable}

From \cref{tab:L_vs_patch_size}, we can see that increasing the patch size beyond a certain point has a detrimental effect
% , as with increasing size,
as the local views tend to lose their granularity and degenerate into global views. Increasing the number of crops has a stronger improvement effect on performance if the patch size is small, thus influencing the value of $k$ accordingly. However, decreasing the patch size at the cost of an increased number of local views also has its downsides - the number of attention computations in the attribute summarization step increases quadratically. Thus $|\mathbb{L}|$ and the local patch size needs to be determined based on application specific accuracy requirements and the available computational resources.

\subsection{Visual Representations of Cross-View Local Relationships}
Our AST-based aggregation scheme allows us to visualize the local relationships that lead to the emergence of the global-view. We aim to construct a graph of local views for depicting the cross-view local relationships. The graph represents the manner in which the local views combine to form the overall object. The nodes of the graph represent the local views. Two nodes are connected via an edge if there exists a relationship between them. The thickness of the edges in the illustration is proportional to the degree of relatedness.
\begin{wrapfigure}[14]{r}{0.6\textwidth}
    \centering
    \includegraphics[height=3cm,width=7cm]{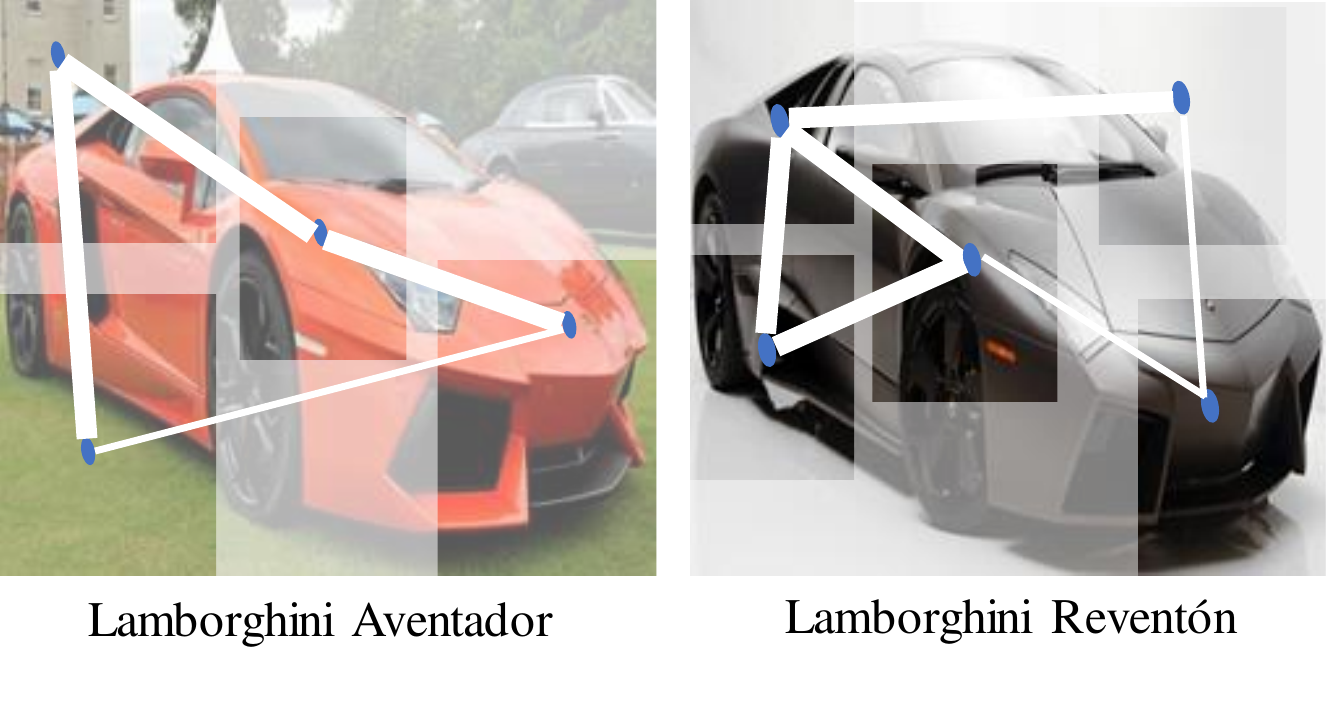}
    \caption{Visualization of the learned relationships across local-views. Despite close similarities between the two car categories, it can be seen that our model leverages discriminative cross-view relationships to tell the instances apart.}
    \label{fig:qualitative_result_main}
\end{wrapfigure}
We compute the topology of this graph by analyzing the final layer mutual attention values of the Attribute Summarization Transformer (AST). We add an edge between two local views if their mutual attention score is higher than a threshold (which we choose to be the average of all pairwise attention scores). The weight of the edge is proportional to the magnitude of attention. For the purpose of simplicity, we depict fewer local views in the visualization, than are actually used for computation. \cref{fig:qualitative_result_main} shows example graphs on images from the Stanford Cars dataset.
In Appendix \ref{sec:qual_res}, we provide more such qualitative results and based on these graphs, we provide an analysis of scenarios under which even relational information cannot distinguish between certain fine-grained categories.

\section{Conclusion and Discussion}
\label{sec:concl}
Starting with the idea of $k$-distinguishability, we derived the necessary and sufficient conditions that a model must satisfy in order to completely capture the fine-grained information in an image. We proved that a model needs to simultaneously encode both view-specific and cross-view relational properties of an object in order to bridge the information gap that its representations have with the semantic content in the input image. Based on our theoretical findings, we designed Relational Proxies, a method that achieves state-of-the-art results on benchmark FGVC datasets by learning class representations conditioned with cross-view relationships.
% By introducing a theoretically rigorous framework for studying FGVC, we believe our work provides the basis for the development of explainable fine-grained features. Such features can be used for computing a minimal set of fine-grained attributes to limit compute time/resources, or to perform tasks like cross-modal retrieval in domains with large modality gap.
By introducing a theoretically rigorous framework, we believe that our work opens up new avenues for studying the problem of FGVC in a more systematic manner. One immediate potential outcome of our work that we foresee is the development of explainable fine-grained features. Such features can be used for computing a minimal set of fine-grained attributes to limit compute time/resources, or to perform tasks like cross-modal retrieval in domains with large modality gap.

\myparagraph{Limitations} -- The process of obtaining local views in our method is somewhat of an uninformed, generic cropping methodology on the global view of the object, which may not necessarily always yield the best set of local object parts. More informed ways of detecting novel object parts from which the global view emerges can lead to obtaining at par performance but with fewer local views.

\myparagraph{Societal Impacts} -- The rigorous theoretical basis of our work has a positive societal impact, which not only makes our methodology transparent and easy to analyze, but also provides a framework to study the foundations of FGVC in general.
So far, we are not aware of any negative societal impact that is specific to our methodology. However, as with all data-driven approaches, underlying biases in the datasets on which our model is trained would influence the patterns learned by it.

\section*{Acknowledgements}
This work has been partially supported by the ERC 853489--DEXIM, by the DFG--EXC number 2064/1--Project number 390727645, and as part of the Excellence Strategy of the German Federal and State Governments.

\clearpage

\bibliographystyle{ieee_fullname}
\bibliography{bibliography}

\section*{Checklist}

\begin{enumerate}

\item For all authors...
\begin{enumerate}
  \item Do the main claims made in the abstract and introduction accurately reflect the paper's contributions and scope?
    \answerYes{}
  \item Did you describe the limitations of your work?
    \answerYes{}
  \item Did you discuss any potential negative societal impacts of your work?
    \answerYes{}
  \item Have you read the ethics review guidelines and ensured that your paper conforms to them?
    \answerYes{}
\end{enumerate}

\item If you are including theoretical results...
\begin{enumerate}
  \item Did you state the full set of assumptions of all theoretical results?
    \answerYes{}
        \item Did you include complete proofs of all theoretical results?
    \answerYes{}
\end{enumerate}

\item If you ran experiments...
\begin{enumerate}
  \item Did you include the code, data, and instructions needed to reproduce the main experimental results (either in the supplemental material or as a URL)?
    \answerNA{Code and pre-trained models will be made public upon paper acceptance. Details of all experimental settings required to reproduce our results are provided in \cref{subsec:expt_settings_datasets}.}
  \item Did you specify all the training details (e.g., data splits, hyperparameters, how they were chosen)?
    \answerYes{}
        \item Did you report error bars (e.g., with respect to the random seed after running experiments multiple times)?
    \answerNo{But the numbers we report are the means of 5 runs with different random seeds.}
        \item Did you include the total amount of compute and the type of resources used (e.g., type of GPUs, internal cluster, or cloud provider)?
    \answerYes{}
\end{enumerate}

\item If you are using existing assets (e.g., code, data, models) or curating/releasing new assets...
\begin{enumerate}
  \item If your work uses existing assets, did you cite the creators?
    \answerYes{}
  \item Did you mention the license of the assets?
    \answerNA{}
  \item Did you include any new assets either in the supplemental material or as a URL?
    \answerNo{}
  \item Did you discuss whether and how consent was obtained from people whose data you're using/curating?
    \answerNA{}
  \item Did you discuss whether the data you are using/curating contains personally identifiable information or offensive content?
    \answerNA{}
\end{enumerate}

\item If you used crowdsourcing or conducted research with human subjects...
\begin{enumerate}
  \item Did you include the full text of instructions given to participants and screenshots, if applicable?
    \answerNA{}
  \item Did you describe any potential participant risks, with links to Institutional Review Board (IRB) approvals, if applicable?
    \answerNA{}
  \item Did you include the estimated hourly wage paid to participants and the total amount spent on participant compensation?
    \answerNA{}
\end{enumerate}

\end{enumerate}

\newpage

\section{Appendix}
\label{sec:appendix}

\subsection{Properties of the Relationship Modelling Function}
\label{subsec:props_of_xi}
\myparagraph{Intuitive Analogy:} The problem of local-to-global relation computation can be viewed as a bit-string-to-integer matching problem.
Consider 3 bits, say $b_1, b_2$ and $b_3$, corresponding to 3 local views. Let the global view be represented by an integer that can be encoded with 3 bits, say with a value of $g$ = 6, for this example. The problem then is to find the association of the integer 6 with its corresponding binary representation of 110. This association represents the cross-view relationship.

The first step towards solving this problem is to \emph{enumerate} all the possible ways in which the local views can combine (to produce any global view, not specifically g). The set of all such combinations will be given by $S = \{000, 001, 010,..., 110, 111\}$. The bit values encode the presence or absence of a particular view in the cross-view relationship. So, no matter what order we observe $b_1, b_2$ and $b_3$ in, we must output the same set $S$, as it is required to be an exhaustive enumeration. This is exactly what the property of permutation invariance achieves.
Once we have S, the next step is to \emph{find} the mapping $S, g \mapsto 110$, i.e, the correct binary encoding for the integer $g = 6$, which is accomplished by the property of view-unification.

\myparagraph{Purpose:} As illustrated through the above analogy, one can view the local-to-global relationship modelling function as an enumerative search algorithm - given a set of local views, it first \emph{enumerates} all possible ways in which they can combine to form a meaningful global view. Given that enumeration, it then \emph{finds} the target solution by learning to identify the correct combination that matches with the global-view representation. Thus, the enumerate operation needs to be permutation invariant, as it has to consider all possible combinations of the inputs, and the find operation needs to be a view-unifier by construction.

\myparagraph{Motivation:} Behind our specific design choice was the motivation to keep the enumerate and find steps separate. This allows the model to have dedicated representation spaces for the two distinct sub-tasks, which in turn facilitates better convergence.

\begin{figure}[!ht]
    \centering
    % \resizebox{\textwidth}{!}{
    \includegraphics[height=4cm]{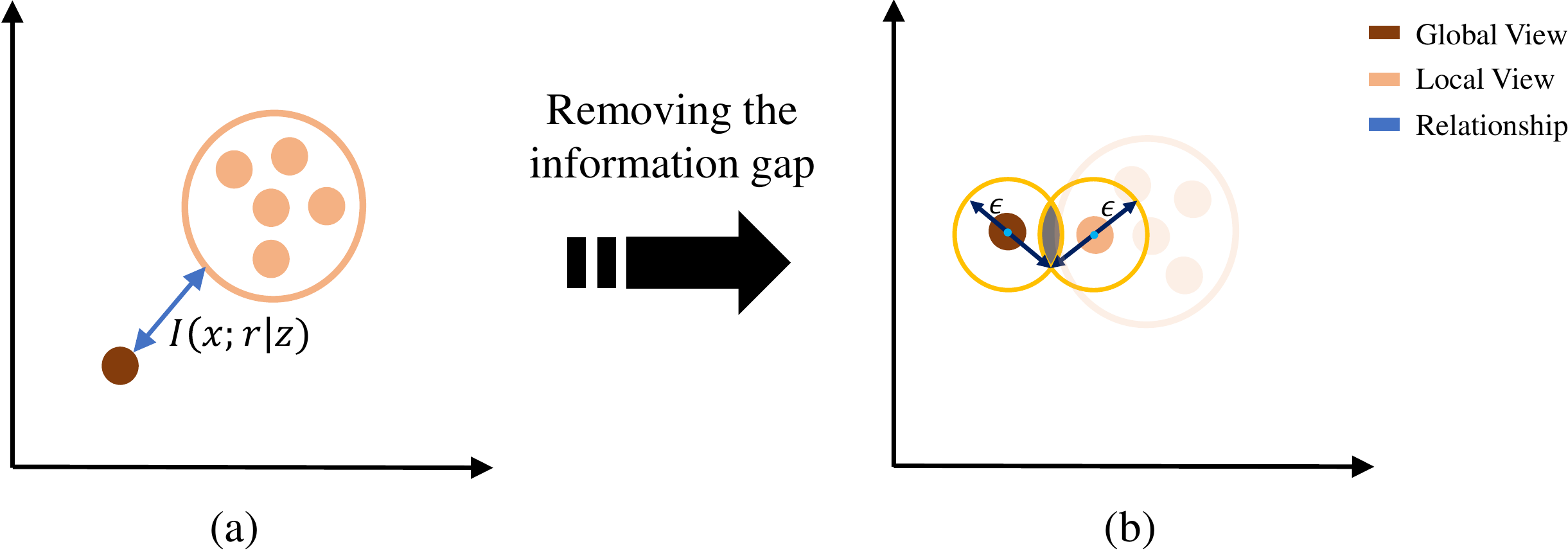}
    % }
    \caption{(a) A relation-agnostic representation space. (b) The $\epsilon$-neighborhoods of the global and local views begin colliding as the information gap is reduced.}
    \label{fig:geom_rel_agn}
\end{figure}

\subsection{Proofs of Additional Identities}
\begin{identity}
    Given a relation-agnostic representation $\vectorname{z}$ of $\vectorname{x}$, the only uncertainty that remains about the label information $\vectorname{y}$ can be quantified as the cross-view relational information $\vectorname{r}$, \ie, $I(\vectorname{x};\vectorname{y} | \vectorname{z}) = I(\vectorname{x};\vectorname{r})$.
    \label{identity:relational_uncertainty_z}
\end{identity}

\begin{proof}
    Using the chain rule for mutual information \cite{Federici2020MIB}, we can factorize the label information $\vectorname{y}$ contained in $\vectorname{x}$, \ie, $I(\vectorname{x}; \vectorname{y})$ as:
    \begin{equation}
        I(\vectorname{x}; \vectorname{y}) = I(\vectorname{x};\vectorname{y}|\vectorname{z}) + I(\vectorname{x};\vectorname{z})
        \label{Ixy_chainRule}
    \end{equation}

% \massi{This "since" to me is not super clear. The equation below (\cref{eqn:Ixy_views}) might be used to explain this part.}
%   Since the label information in $\vectorname{x}$ can be expressed exclusively as a function of its global ($\vectorname{g}$) and local ($\vectorname{l}_i$) views \cite{Zhang2021MMAL}, along with the cross-view relationship $\vectorname{r}$ \cite{Choudhury2021UnsupervisedParts}, we can also write $I(\vectorname{x}; \vectorname{y})$ as:
As evidenced by recent literature \cite{Luo2019CrossX, Zhang2021MMAL, Behera2021CAP, Choudhury2021UnsupervisedParts}, the label information in $\vectorname{x}$ can be expressed exclusively as a function of its global ($\vectorname{g}$) and local ($\vectorname{l}_i$) views. Thus, in quantitative terms, the label information $\vectorname{y}$ in $\vectorname{x}$ can also be factorized into \emph{relation-agnostic} and \emph{relation-aware} components as follows:
% \anjan{needs an evidence} \massi{+ better to explain the terms and why we consider them to be independent (also, what we mean by independence, since local and global views are anyway intertwined}:
\begin{equation}
    I(\vectorname{x}; \vectorname{y}) = \underbrace{I(\vectorname{x};\vectorname{g}) + \sum\limits_{\vectorname{l} \in \mathbb{L}} I(\vectorname{x};\vectorname{l})}_\text{relation-agnostic} + \underbrace{I(\vectorname{x};\vectorname{r})}_\text{relation-aware}
    \label{eqn:Ixy_views}
\end{equation}

% \massi{Do we have in the objective anything that makes the representations orthogonal? I am asking because, otherwise, we need to be careful in the statements. For instance, intuitively, the mutual information between a global and local representation is not zero, since one might infer the local representation given the global one.}\change{Changed ``orthogonal'' to ``disjoint''.}
The relation-aware representation $\vectorname{r}$ is, unlike relation-agnostic representations, obtained explicitly based on the cross-view relationship. However, since $f$ computes $\vectorname{z}$ without considering any relational information, it only models the relation-agnostic component of \cref{eqn:Ixy_views}. Thus,
% However, since $\vectorname{z}$ is a relation-agnostic representation, it only models the relation-agnostic component of \cref{eqn:Ixy_views} without considering the relational part. Thus,
\begin{equation}
    I(\vectorname{x};\vectorname{z}) = I(\vectorname{x};\vectorname{g}) + \sum\limits_{\vectorname{l} \in \mathbb{L}} I(\vectorname{x};\vectorname{l})
    \label{eqn:Ixz_disjoint}
\end{equation}

% \massi{In Eq.~\eqref{eqn:Ixy_views}, we do not have $I(\vectorname{x};\vectorname{z})$ (i.e. it is already substituted at the moment).}
Substituting the \emph{relation-agnostic} component of \cref{eqn:Ixy_views} with the L.H.S. of \cref{eqn:Ixz_disjoint}, and comparing it with \cref{Ixy_chainRule}, we get:
% Substituting \cref{eqn:Ixz_disjoint} in \cref{eqn:Ixy_views} and comparing it with \cref{Ixy_chainRule}, we get:
\begin{equation}
    I(\vectorname{x};\vectorname{y} | \vectorname{z}) = I(\vectorname{x};\vectorname{r})
\end{equation}
\end{proof}

\subsection{Geometric Relation Agnosticity}
\label{subsec:note_geom_rel_agn}
\cref{def:geometric-disjointness} is based on the fact that the information gap (derived in \cref{prop:Ixy_Izy_gap}) between the global and the local views has the effect that the two view families would be mapped to distinct locations in the representation space, and the separation between them would be proportional to the information gap, \ie, $I(\vectorname{x}; \vectorname{r} | \vectorname{z})$. \cref{def:geometric-disjointness} also mentions that relation-agnostic embeddings of the local and the global views must thus be well separated, \ie, the $\epsilon$-neighborhood of the global embedding $n_\epsilon(\vectorname{z}_g)$ must not intersect with those of the local embeddings $n_\epsilon(\vectorname{z}_l)$. In other words, the global embedding must be sufficiently far apart from each of the local embeddings.

\cref{fig:geom_rel_agn} depicts the geometric effect of removing the information gap from a relation-agnostic representation space. As proven in \cref{lemma:fLearnsDisjoint}, if the information gap is reduced using the same encoder $f$ that was used to obtain $\vectorname{z_g}$ and $\vectorname{z}_l$, the model starts mapping the global and the local views to identical regions in the representation space. This could potentially lead to the requirement of $k$-distinguishability to not be satisfied, as the unique information pertaining to at least one of the local views is lost upon merger with the global view (and vice-versa). It is thus a requirement for a sufficient learner to preserve the relation-agnosticity in the representation space of $f$.

\subsection{Relation-Agnosticity of Relational Proxies}

\begin{figure}[!t]
    \centering
    % \resizebox{\textwidth}{!}{
    \includegraphics[height=5cm,width=0.8\textwidth]{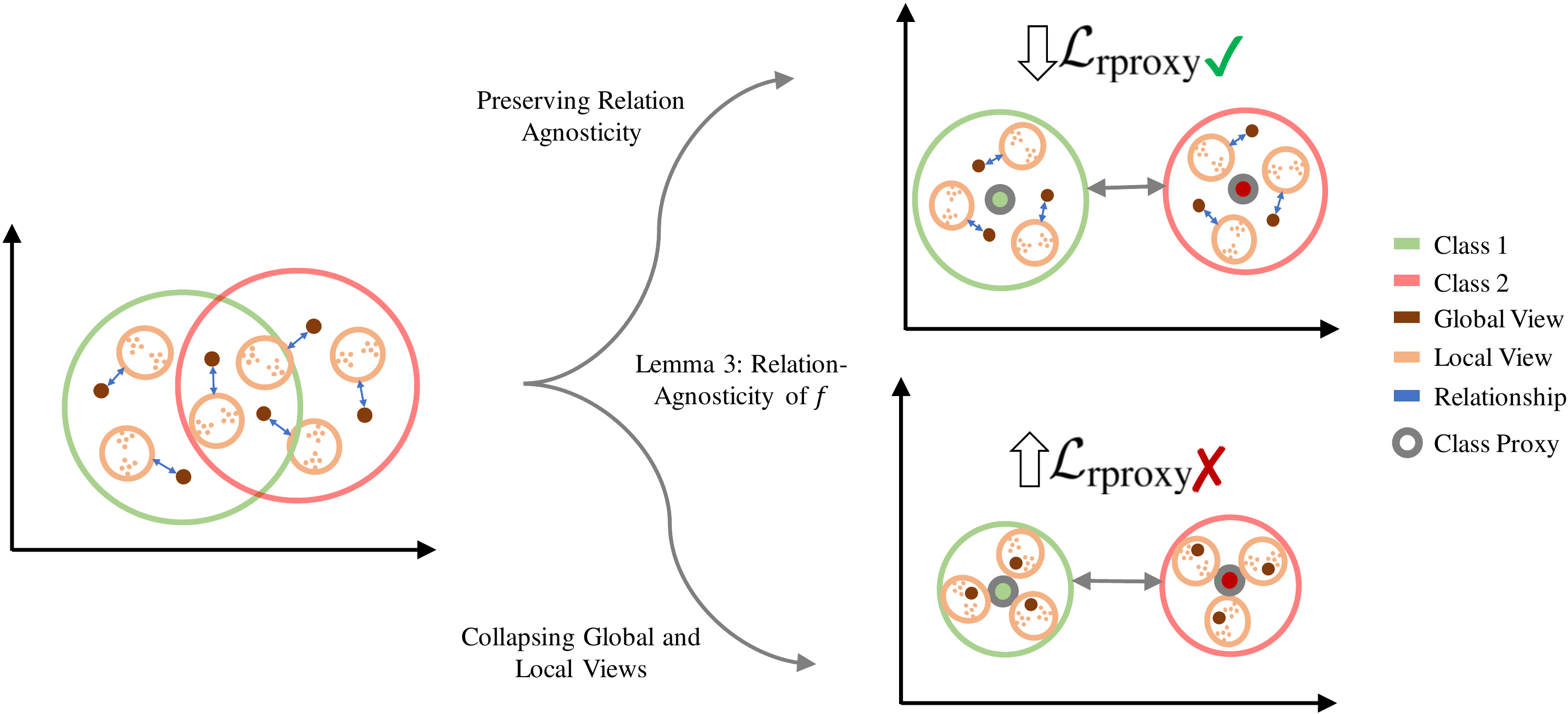}
    % }

    \caption{
    Embeddings of datapoints from two classes obtained from $f$ before training (left). As $f$ is trained with an end objective of minimizing $\mathcal{L}_\text{rporxy}$, it has two potential choices (right). However, as proven in \cref{lemma:fLearnsDisjoint}, the relation-agnostic nature of $f$ prevents the collapse of the global and local embeddings even when they share the same set of class proxies.
    % \cref{lemma:fLearnsDisjoint} requires the global and local representations to maintain the information gap (equal to the relational information) produced by the relation-agnosticity of $f$ when the downstream objective is to minimize a cross-entropy loss with the true class distribution $\vectorname{y}$ (which happens to be $\mathcal{L}_\text{rproxy}$ in our case).
    }
    \label{fig:f_rel_agn}
\end{figure}

\label{subsec:proxy_relation_agnosticity)}
% INCLUDE IN APPENDIX
% 
The representations $\vectorname{z}_g$ and $\vectorname{z}_{l_i}$ are computed in a relation-agnostic manner and no explicit operation is performed to reduce the domain gap between the global and the set of local views. This natural domain gap thus manifests in the representation space of $\vectorname{z}$ as its relation-agnostic nature.
% (proven in \cref{lemma:fLearnsDisjoint}).

% Contrasting views with proxies leads to a contrastive learning across classes and not across views of the same class. Non-alignment of different view families (\ie, local and global) is what the criterion of relation-agnosticity requires. So, metric learning on the local and global views using the same set of proxies will not violate the relation-agnosticity criterion. Independent encoding of the views still preserves relation-agnosticity as no cross-view relationship is explicitly mined.
% 
\cref{fig:f_rel_agn} diagrammatically illustrates this idea.
% The relation-agnosticity between the global and the local views is still preserved, even though they may be mapped to the same class/proxy neighborhood. If the global and the local embeddings were collapsed to $\epsilon$-neighborhoods of each other, then the representation space would lose its relation-agnostic nature and the two classes would no longer remain distinguishable according to the $k$-distinguishability criterion. We prove via \cref{lemma:fLearnsDisjoint} that this is not possible if the end objective is to reduce the cross-entropy loss between the ground-truth and predicted class distributions for an instance of $\mathcal{P}_\text{FGVC}$.
Given an entangled representation space where the classes are not entirely separable (left), the encoder $f$ has two choices to map the local and global views of the corresponding datapoints to completely separable proxy neighborhoods.
It could either:
\begin{enumerate}
    \item Preserve the relation-agnosticity by maintaining the information gap (equal to the cross-view relational information) even within the proxy neighborhood (top right), or
    \item Collapse the local and global representations in the process of alignment (bottom right) by mapping them to $\epsilon$-neighborhoods of each other.
\end{enumerate}
However, since the end objective of our model is to minimize $\mathcal{L}_\text{rproxy}$, which is cross-entropic in nature, we prove via \cref{lemma:fLearnsDisjoint} that $f$ cannot collapse the local and global representations, as that would lead to an increase in the downstream cross-entropy loss. $f$ would thus choose to preserve the relational gap in the representation space while mapping them to the neighborhood of their corresponding proxy.

\begin{figure}[!t]
    \centering
    \resizebox{0.9\textwidth}{!}{
    \includegraphics{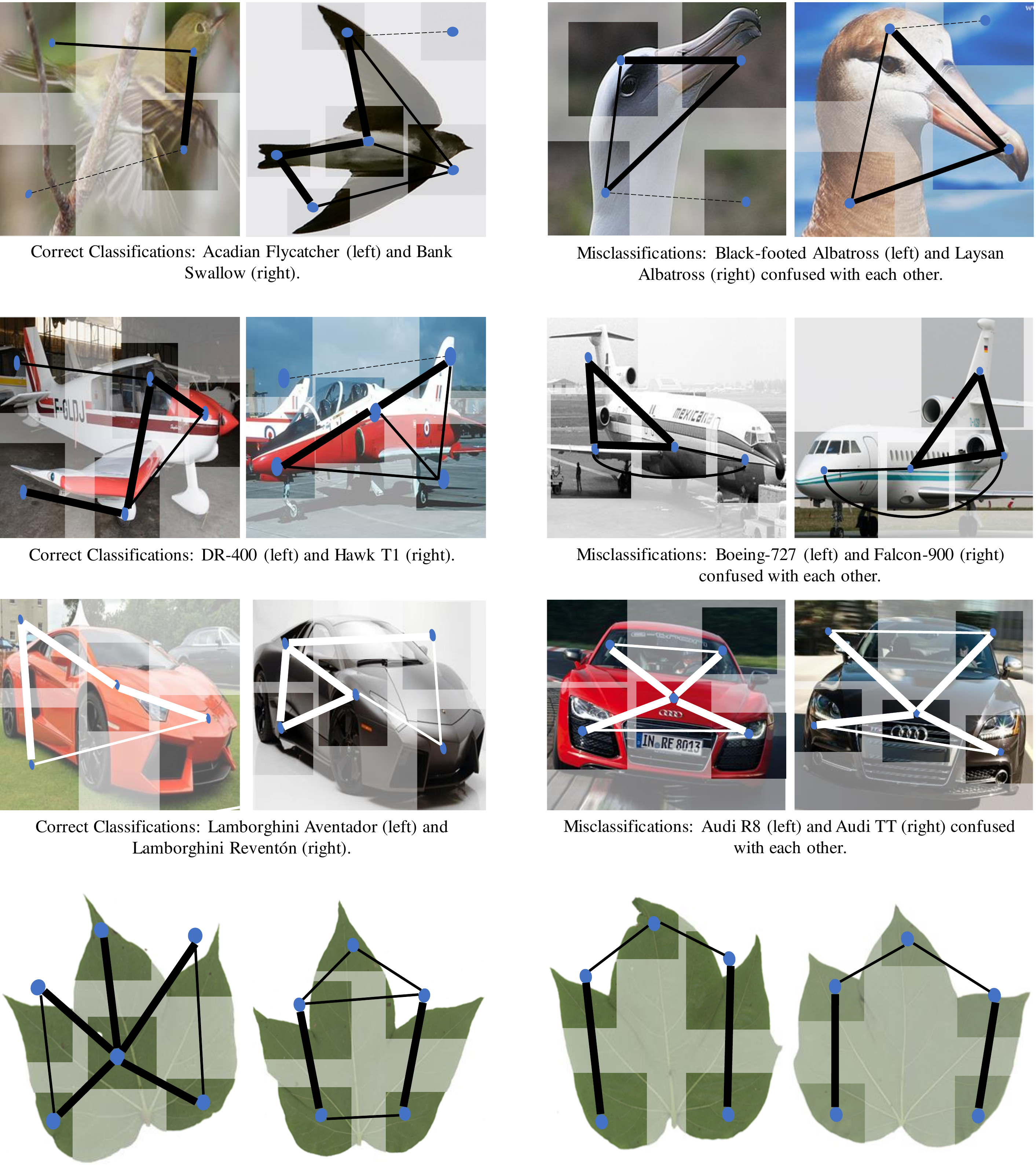}}
    \caption{Visual representations of cross-view relationships along with qualitative classification results on (in order from top) CUB, FGVC Aircraft, Stanford Cars and Cotton Cultivar datasets. The pairs on the left correspond to correct classifications made by our model, while the ones on the right are misclassifications occurring out of cross-category confusions.}
    \label{fig:qual_class_res}
\end{figure}

\subsection{Visual Representations of Cross-View Local Relationships}
\label{sec:qual_res}
% We aim to construct a graph of local views for depicting the cross-view local relationships.
% The cross-view relationships are depicted via a graph of the local views.
% The graph represents the manner in which the local views combine to form the overall object. The nodes of the graph represent the local views. Two nodes are connected via an edge if there exists a relationship between them. The thickness of the edges in the illustration is proportional to the degree of relatedness. The way we computed the topology of this graph was by analyzing the final layer mutual attention values of the Attribute Summarization Transformer (AST). We add an edge between two local views if their mutual attention score is higher than a threshold (which we choose to be the average of all pairwise attention scores). The weight of the edge is proportional to the magnitude of attention. For the purpose of simplicity, we depict fewer local views in the visualization, than are actually used for computation.

% \myparagraph{Observations}:
\cref{fig:qual_class_res} depicts examples of graphs depicting cross-view local relationships. It can be seen that images that provide a diverse set of local views, and thus, a larger space of possible cross-view relationships are the ones that get classified correctly with full certainty. However, as the number of unique local views get limited (possibly due to occlusion or an incomplete photographing of the object), it reduces the amount of relational information that can be mined. Under situations when even the individual local-views are largely shared between classes, there remains no discriminative premise (neither local/global, nor relational) for telling their instances (with limited depiction of local views) apart. It is under such circumstances that the classifier gets confused.

\myparagraph{Example}: For instance, in the example from the CUB dataset (the top row in \cref{fig:qual_class_res}), the images of the Acadian Flycatcher and Bank Swallow depict sufficient numbers of local views like the head, tail, belly and wings, which provide a large space of potential cross-view relationships that favor classification outcome. On the other hand, the images of the Black-footed Albatross and Laysan Albatross only depict the head and the neck, thus limiting the number of computable relationships that can act as discriminators. Moreover, the head and the neck look largely similar between the two categories, thereby leading to cross-category confusion causing a subsequent misclassification. However, we believe that such a situation can be addressed by learning different distributional priors over the set of local views, which we plan to take up as future work.

\newpage

\section{Supplementary}

\subsection{Additional Experiments}

\myparagraph{Fine-grained performance boost on ImageNet subsets over SotA}
We compare our method to TransFG \cite{He2022TransFG}, the SotA FGVC method on Dogs ImageNet. We summarize our findings in \cref{tab:coarse_vs_fine_grained_sota}, which shows that our method provides state-of-the-art performance boost in the fine-grained setting over vanilla relation-agnostic encoders. $\Delta_1$ and $\Delta_2$ denote the perfomance boost achieved by an FGVC method over relation-agnostic encoders in the coarse-grained and fine-grained settings respectively. 
% \anjan{(1) Do not write such a long sentence within parenthesis. (2) Do not only write the observations. Emphasise impact of our proposal.}

Relational features play a much more significant role in distinguishing fine-grained categories than coarse-grained ones. This is because most coarse-grained classes can be distinguished by local/global features alone, and would not require relational information. However, for fine-grained classes, the cross-view relationships often happen to be the only available discriminator. Thus, a learner not leveraging the same would suffer from the information gap (Section 3.2 in the main manuscript), not providing any significant boost over a relation-agnostic encoder. Our method, by capturing the cross-view relationships, is able to bridge this information gap.

\begin{table}[!ht]
    \centering
    % \resizebox{\textwidth}{!}{
    \begin{tabular}{c|c|c|c|c|c}
        \hline
         Method & Tiny ImageNet & $\Delta_1$ & Dogs ImageNet & $\Delta_2$ & $\Delta_2 - \Delta_1$ \\
         \hline
         Relation-Agnostic Encoder & 88.75 & & 91.30 & & \\
         \hline
         TransFG \cite{He2022TransFG} & 88.85 & 0.10 & 92.30 & 1.00 & 0.90 \\
         Relational Proxy (Ours) & 88.91 & 0.16 & 92.75 & 1.45 & \textbf{1.29} \\
         \hline
    \end{tabular}
    % }
    \caption{Comparison of coarse \textit{vs.} fine-grained accuracy gains over a relation-agnostic encoder.}
    \label{tab:coarse_vs_fine_grained_sota}
\end{table}

\myparagraph{Permutation invariance of AST}
For our method to be robust to changes in pose and relative orientation of local object parts, we require the Attribute Summarization Transformer (AST) to be permutation invariant. We achieve the same by eliminating position embeddings \cite{Naseer2021IntriguingViT} from our AST.
We test the validity of our requirement by comparing the classification accuracy of Relational Proxies having ASTs with and without position embeddings \cite{Naseer2021IntriguingViT}. We summarize our findings in \cref{tab:pos_embed_AST}, which shows that making the AST permutation invariant in fact plays a role in enhancing the performance of our model.
% \anjan{Again, why the difference on benchmark is smaller than the difference on Cultivar?}

Given the low inter-class variation of the cultivar datasets, parts of leaves from different classes could appear the same under changes in orientation, making a permutation sensitive model mistake it for a different class. For this reason, the AST without position embeddings (permutation invariant) performs significantly better (compared to other benchmarks) than the one with position embeddings (permutation sensitive).

\begin{table}[!ht]
    \centering
    \resizebox{\textwidth}{!}{
    \begin{tabular}{c|c|c|c|c|c|c}
         \hline
         \multirow{2}{*}{\textbf{Method}} & \multicolumn{4}{c|}{\textbf{Benchmark}} & \multicolumn{2}{c}{\textbf{Cultivar}} \\
         \cline{2-7}
         % hline
         & \textbf{FGVC Aircraft} & \textbf{Stanford Cars} & \textbf{CUB} & \textbf{NA Birds} & \textbf{Cotton} & \textbf{Soy} \\
         \hline
         w/ Position Embeddings & 95.11 & 96.15 & 91.82 & 91.09 & 68.77 & 50.15 \\
         w/o Position Embeddings & \textbf{95.25} & \textbf{96.30} & \textbf{92.00} & \textbf{91.20} & \textbf{69.81} & \textbf{51.20} \\
         \hline
    \end{tabular}}
    \caption{Effect of position embeddings on the permutation invariance of the Attribute Summarization Transformer (AST).}
    \label{tab:pos_embed_AST}
\end{table}

\myparagraph{Evaluation with VGG-16 Backbone}
To ensure that our method has no backbone specific dependency, we perform evaluations with VGG-16 \cite{Simonyan2015VGG16} backbone and report our findings in \cref{tab:eval_vgg16}. As the numbers show, our method remains stable across backbones, significantly outperforming SotA methods that report performances with VGG-16 backbones as well.

\begin{table}[!ht]
    \centering
    \begin{tabular}{c|c|c}
    \hline
         Method & FGVC Aircraft & CUB \\
         \hline
         MaxEnt \cite{dubey2018MaxEntFGVC} & 78.08 & 77.02 \\
         MMAL \cite{Zhang2021MMAL} & 87.00 & 83.75 \\
         \textbf{Ours (Relational Proxies)} & \textbf{91.20} \textpm 0.03 & \textbf{88.13} \textpm 0.01 \\
         \hline
    \end{tabular}
    \caption{Comparison of our method with state-of-the-art using VGG-16 backbone.}
    \label{tab:eval_vgg16}
\end{table}

\subsection{Qualitative Results}

\myparagraph{Importance of Relational Information}
\cref{fig:umap_global_local_relational} shows examples of classes that cannot be separated by global or local information alone. The cross-view relational information serves as the strongest discriminator for such classes.
For example, Black-footed Albatross, Laysan Albatross and the Sooty Albatross (denoted in red, dark blue and orange respectively), share a large number of local attributes and have similar overall appearances, but have differing geometries. For this reason, as can be observed from the low-dimensional visualization of their embeddings obtained via UMAP \cite{mcinnes2018umap-software}, they are only separable based on their relational features, and not by their global or local features.
Additionally, \cref{fig:umap_across_epochs} shows that such classes becomes separable as the model learns to incorporate the relational information as part of the learning process.

\begin{figure}[!ht]
    \centering
    \resizebox{\textwidth}{!}{
    \includegraphics{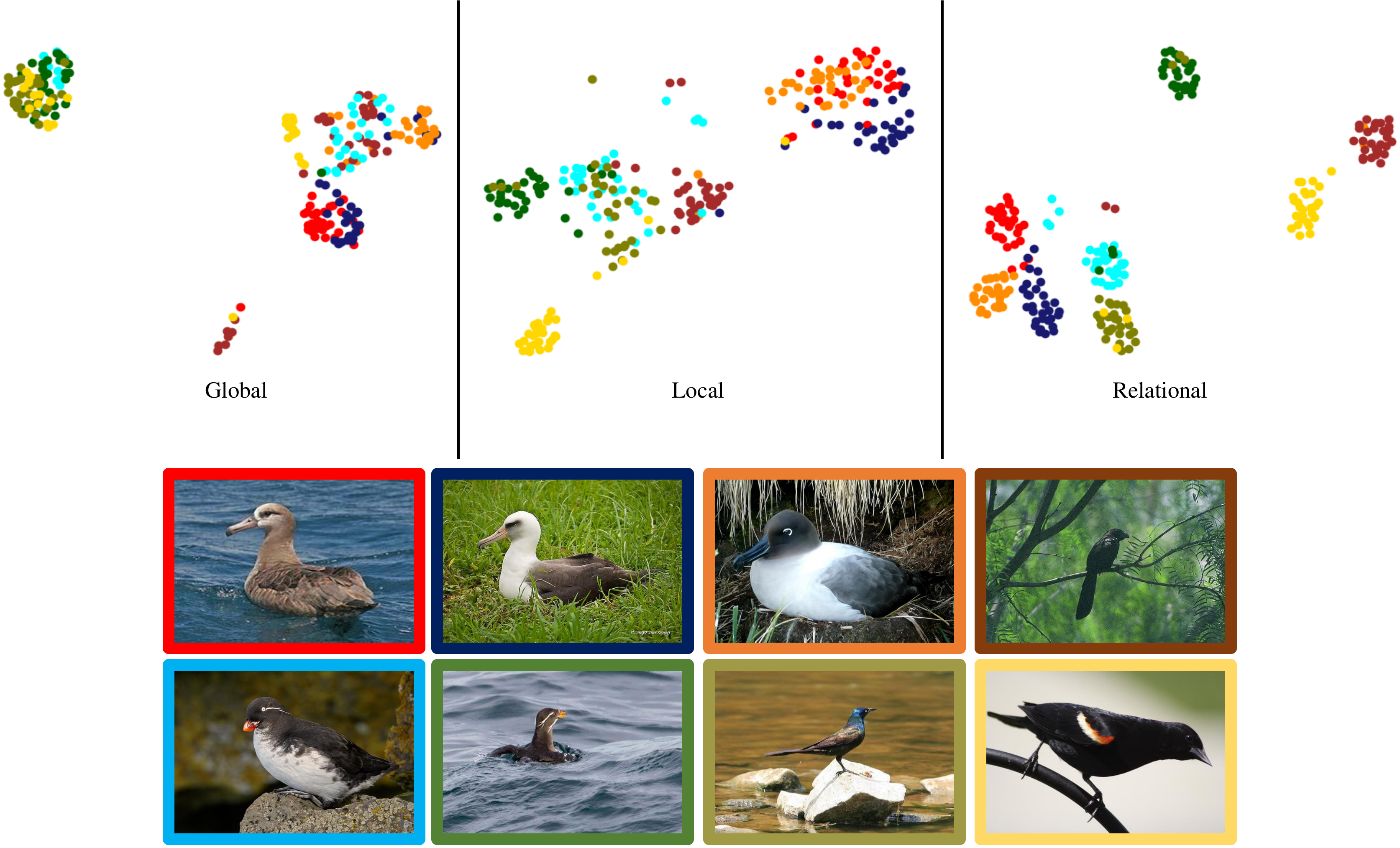}}
    \caption{Top: Low dimensional embedding visualization of categories that are difficult to separate by global or local features alone, but can be separated using relational information. Bottom: Sample images from such categories. Colors indicate category memberships.}
    \label{fig:umap_global_local_relational}
\end{figure}

\begin{figure}[!ht]
    \centering
    \resizebox{\textwidth}{!}{
    \includegraphics{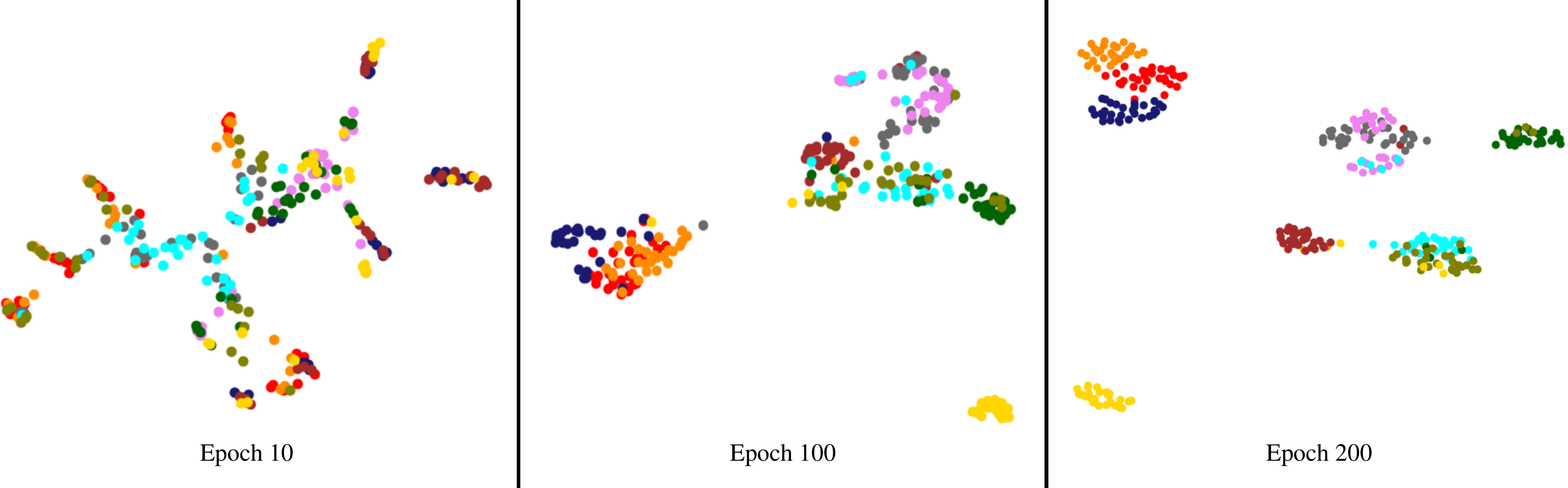}}
    \caption{Low dimensional visualization of the relational representation ($\vectorname{r}$) space evolution across epochs. Colors indicate category memberships.}
    \label{fig:umap_across_epochs}
\end{figure}

\myparagraph{Relation-Agnosticity of Relational Proxies}
\cref{fig:rel_agnosticity} shows UMAP visualizations of global and local embeddings for instances of a single class, obtained from a fully trained Relational Proxy model. It provides empirical evidence for our theoretical result in Lemma 3, \ie, $f$ will produce relation-agnostic representations if the downstream objective is cross-entropic in nature. As can be seen, despite using the same set of proxies for the global and the local views, they get mapped to disjoint locations in the representation space. The distance between the clusters of global and local views is proportional to the information gap (Proposition 1), which is separately being learned by the relational encoder $\xi$ (Proposition 2). However, some global embeddings can still be seen to overlap with the cluster of the locals. This happens with images for which the information provided by the global view becomes redundant after collectively knowing the set of local views. The global view does not provide any additional information and thus can be merged with the local views with no information loss (while maintaining the requirement of $k$-distinguishability).

\begin{figure}
    \centering
    \includegraphics[width=10cm]{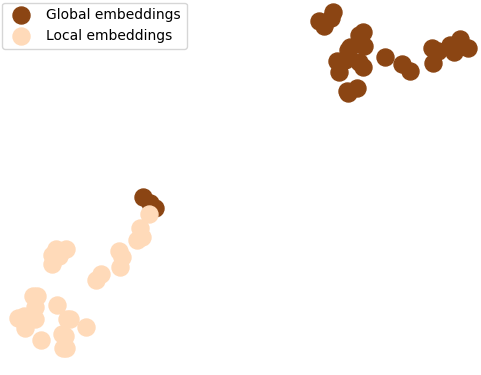}
    \caption{Low dimensional visualization of embeddings of global and local views for instances of a single class. The gap between the two clusters indicate the retention of relation-agnosticity even after the convergence of Relational Proxies, thereby supporting Lemma 3. 
    % \anjan{Why do some global embeddings fall within the neighborhood of local embeddings?}
    }
    \label{fig:rel_agnosticity}
\end{figure}

\begin{algorithm}
    \SetKwInOut{Input}{Input}\SetKwInOut{Output}{Output}
    \Input{A set of images $\mathbb{X}$, their corresponding labels $\mathbb{Y}$, the number of fine-grained categories $c$, the number of epochs $N$, and the learning rate $\eta$.}
    \Output{A relation agnostic-encoder $f$, a cross-view relation encoder $\xi$, and a set of $c$ relational-proxies $\mathbb{P}$ corresponding to the unique labels in $\mathbb{Y}$.
    }
    \tcc{Initialize $c$ learnable class-proxy vectors representing the labels in $\mathbb{Y}$.
    An image with label $\vectorname{y}_i$ has $\vectorname{p}_i$ as its corresponding class-proxy.
    }
    $\mathbb{P} \gets \{\vectorname{p}_1, \vectorname{p}_2, ... \; \vectorname{p}_c\}$ \\
    \For{epoch $\leftarrow 1$ \KwTo $N$}{
    $\mathcal{L}_\text{rproxy} \gets 0$ \\
        \For{$\vectorname{p} \in \mathbb{P}$}{
            $\psi^+ \gets 0; \psi^- \gets 0$ \\
            \For{$\vectorname{x} \in \mathbb{X}$}{
                $\vectorname{g} \gets c_g(\vectorname{x})$ \\
                $\mathbb{L} \gets \{\vectorname{l}_1, \vectorname{l}_2,... \: \vectorname{l}_k\} \gets c_l(\vectorname{x})$ \\
                $\vectorname{z}_g \gets f(\vectorname{g})$ \\
                $\mathbb{Z_L} \gets \{\vectorname{z}_{l_1}, \vectorname{z}_{l_2},... \: \vectorname{z}_{l_k}\} \gets \{f(\vectorname{l}) : \vectorname{l} \in \mathbb{L}\}$ \\
                $\vectorname{z}_\mathbb{L} \gets \operatorname{AST}(\mathbb{Z_L})$ \\
                $\vectorname{r} \gets \rho(\vectorname{z}_g, \vectorname{z}_\mathbb{L})$ \\
                \tcp{true proxy for $\vectorname{x}$}
                \uIf{$\vectorname{p} == \vectorname{p}^+$} {
                    $\psi^+ \gets \psi^+ + e^{-\alpha(s(\vectorname{z}_g, \vectorname{p}) - \delta)} + 
                    e^{-\alpha(s(\vectorname{z}_\mathbb{L}, \vectorname{p}) - \delta)} +
                    e^{-\alpha(s(\vectorname{z}_g, \vectorname{p}) - \delta)}$
                }
                \tcp{negative proxies for $\vectorname{x}$}
                \uElse{
                    $\psi^- \gets \psi^- + e^{\alpha(s(\vectorname{z}_g, \vectorname{p}) + \delta)} + e^{\alpha(s(\vectorname{z}_\mathbb{L}, \vectorname{p}) + \delta)} +
                    e^{\alpha(s(\vectorname{r}, \vectorname{p}) + \delta)}$
                }
            }
            $\psi^+ \gets 1 + \psi^+$ \\
            $\psi^- \gets 1 + \psi^-$ \\
            $\mathcal{L}_\text{rproxy} \gets \mathcal{L}_\text{rproxy} - \frac{1}{c} \log \left( \frac{1}{\psi^+ \cdot \psi^-} \right)$
        }
        $f \gets f - \eta \nabla_f \mathcal{L}_\text{rproxy}$ \\
        $\mathrm{AST} \gets \mathrm{AST} - \eta \nabla_\mathrm{AST} \mathcal{L}_\text{rproxy}$ \\
        $\rho \gets \rho - \eta \nabla_\rho \mathcal{L}_\text{rproxy}$ \\
        \For{$\vectorname{p} \in \mathbb{P}$}{
            $\vectorname{p} \gets \vectorname{p} - \eta \nabla_\vectorname{p} \mathcal{L}_\text{rproxy}$ \\
        }
    }
    \caption{\textsc{Relational-Proxies}: End-to-end training procedure for Relational Proxies.}
    \label{alg:rproxy}
\end{algorithm}

\subsection{Additional notes on Relational Proxies}

\myparagraph{Pseudocode}
\cref{alg:rproxy} provides the pseudocode for training our Relational Proxies model. We start by initializing a set of $c$ learnable class-proxies $\{\vectorname{p}_1, \vectorname{p}_2, ... \; \vectorname{p}_c\}$. For each image $\vectorname{x}$, we obtain its global $\vectorname{z}_g$ and set $\mathbb{Z_L}$ of local representations by propagating their corresponding views (obtained via cropping functions $c_g$ and $c_l$) through a relation-agnostic encoder $f$ (lines 10-11). We then realize the cross-view relational encoder $\xi$ as a combination of the Attribute Summarization Transformer (AST) and the MLP $\rho$. The AST returns a summary of the local views $\vectorname{z}_\mathbb{L}$ (line 12). Using $\vectorname{z}_g$ and $\vectorname{z}_\mathbb{L}$, $\rho$ computes the cross-view relation embedding $\vectorname{r}$ (line 13). Thereafter, all three representation of $\vectorname{x}$, \ie, $\vectorname{z}_g$, $\vectorname{z}_l$ and $\vectorname{r}$ are used to condition the learning of the class proxies. The representations are incentivised to remain close to the proxy corresponding to their true class, while being distant from proxies of other classes (lines 15-19). How far the representation space deviates from this structural requirement is captured by computing the cross-entropic loss $\mathcal{L}_\text{rproxy}$. Minimizing $\mathcal{L}_\text{rproxy}$ thus has the effect of enforcing the representations to form a metric space (lines 23-27). Upon convergence, $\{\vectorname{p}_1, \vectorname{p}_2, ... \; \vectorname{p}_c\}$ serve as the set of Relational Proxies.

\begin{figure}
    \centering
    \begin{subfigure}{0.4\textwidth}
        \centering
        \includegraphics[height=5cm,width=\linewidth]{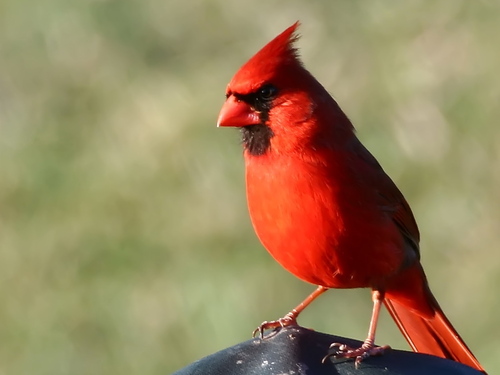}
    \end{subfigure}%
    \hspace{10pt}
    \begin{subfigure}{0.4\textwidth}
        \centering
        \includegraphics[height=5cm, width=\linewidth]{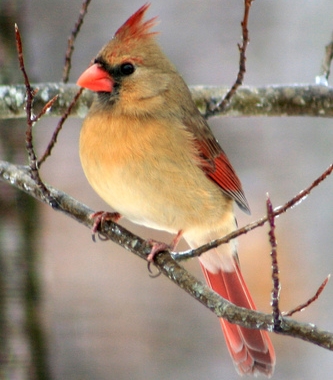}
    \end{subfigure}
    \caption{Male (left) and female (right) cardinals.}
    \label{fig:intraclass_cardinals}
\end{figure}

\myparagraph{Cross-view relationships for intra-class variations}
\cref{fig:intraclass_cardinals} depicts the large variation in non-relational features like color and texture between male and female cardinals. Even though they belong to the same fine-grained category of cardinal birds, a model not accounting for the relationships between the individual local parts and the global view of the object, and hence not capturing the fine-grained geometric relationships, would not be able to map such significantly varying instances to the same neighborhood of the representation space. In such scenarios, the relational information becomes the only component that can be used to learn compact representations of categories with such large intra-class variations.

\end{document}